\documentclass{article}

\usepackage[accepted]{icml2023}

\usepackage{xcolor}
\usepackage[colorinlistoftodos,prependcaption,textsize=tiny]{todonotes}
\newcounter{fs}

\newcounter{mf}
\usepackage{amsmath,amsfonts,bm}
\usepackage[utf8]{inputenc} 
\usepackage[T1]{fontenc}    
\usepackage{hyperref}       
\usepackage{url}            
\usepackage{amsfonts}       
\usepackage{nicefrac}       
\usepackage{graphicx}
\usepackage{placeins}
\usepackage{subfigure}
\usepackage{dsfont}
\usepackage{xspace}
\usepackage{amssymb}
\usepackage{mathtools}
\usepackage{hyperref}
\usepackage{amsfonts}
\usepackage{color}

\usepackage{float}
\usepackage{scalerel,stackengine}
\usepackage{soul}
\usepackage[algo2e,ruled,vlined]{algorithm2e}
\include{pythonlisting}
\usepackage{appendix}
\usepackage{amsthm}

\usepackage{wasysym}
\usepackage{pifont}
\usepackage{bbm}
\usepackage{lipsum}

\def\eqref#1{equation~\ref{#1}}

\def\1{\bm{1}}

\DeclareMathAlphabet{\mathsfit}{\encodingdefault}{\sfdefault}{m}{sl}
\SetMathAlphabet{\mathsfit}{bold}{\encodingdefault}{\sfdefault}{bx}{n}

\newcommand{\E}{\mathbb{E}}

\usepackage{hyperref}
\usepackage{url}
\usepackage{booktabs}
\usepackage[super]{nth}
\usepackage{bm}
\usepackage{graphicx}
\usepackage{wrapfig}

\usepackage[font=small,labelfont=bf]{caption}

\hypersetup{ 
    colorlinks=true,
    linkcolor=blue,
    filecolor=magenta,      
    urlcolor=cyan,
    citecolor=blue,
    pdftitle={User-defined Event Sampling and Uncertainty Quantification in Diffusion Models for Physical Dynamical Systems},
    pdfpagemode=FullScreen,
    }

\usepackage[nameinlink,capitalise,noabbrev]{cleveref} 

\usepackage{empheq}
\usepackage[most]{tcolorbox}
\newtcbox{\mymath}[1][]{%
    nobeforeafter, math upper, tcbox raise base,
    enhanced, colframe=black!30!black,
    colback=white!30, boxrule=1pt,
    #1}
\newtcbox{\mywboxtext}{on line,colback=white,colframe=black,size=fbox,arc=3pt,boxrule=0.8pt}
\newcommand{\mywboxmath}[1]{\mywboxtext{$#1$}}

\usepackage{xspace}
\newcommand*{\ie}{i.e.\@\xspace}
\newcommand{\eat}[1]{}

\usepackage{amsthm}
\newtheorem{theorem}{Theorem}

\usepackage{tcolorbox}
\definecolor{lightblue}{HTML}{18282e}
\definecolor{lighterblue}{HTML}{f2fafd}  
\newtcolorbox{abox}{colback=lighterblue,colframe=lightblue}

\usepackage[textsize=tiny]{todonotes}

\icmltitlerunning{User-defined Event Sampling and Uncertainty Quantification in Diffusion Models for Physical Dynamical Systems}

\begin{document}

\twocolumn[
\icmltitle{User-defined Event Sampling and Uncertainty Quantification in Diffusion Models for Physical Dynamical Systems}

\icmlsetsymbol{equal}{*}

\begin{icmlauthorlist}
\icmlauthor{Marc Finzi}{google,nyu,internship}
\icmlauthor{Anudhyan Boral}{google}
\icmlauthor{Andrew Gordon Wilson}{nyu}
\icmlauthor{Fei Sha}{google}
\icmlauthor{Leonardo Zepeda-N\'u\~nez}{google}
\end{icmlauthorlist}

\icmlaffiliation{nyu}{Department of Computer Science, New York University, NYC, USA}
\icmlaffiliation{google}{Google Research, 1600 Amphitheatre Pkwy Mountain View CA 94043, USA}
\icmlaffiliation{internship}{Work done during an internship at Google Research}

\icmlcorrespondingauthor{Marc Finzi}{maf820@nyu.edu}

\icmlkeywords{Machine Learning, ICML}

\vskip 0.3in
]



\printAffiliationsAndNotice{}  

\begin{abstract}
Diffusion models are a class of probabilistic generative models that have been widely used as a prior for image processing tasks like text conditional generation and inpainting. We demonstrate that these models can be adapted to make predictions and provide uncertainty quantification for chaotic dynamical systems. In these applications, diffusion models can implicitly represent knowledge about outliers and extreme events; however, querying that knowledge through conditional sampling or measuring probabilities is surprisingly difficult. Existing methods for conditional sampling at inference time seek mainly to enforce the constraints, which is insufficient to match the statistics of the distribution or compute the probability of the chosen events. To achieve these ends, optimally one would use the conditional score function, but its computation is typically intractable. In this work, we develop a probabilistic approximation scheme for the conditional score function which provably converges to the true distribution as the noise level decreases. With this scheme we are able to sample conditionally on nonlinear user-defined events at inference time, and matches data statistics even when sampling from the tails of the distribution.
\end{abstract}

\section{Introduction}

Accurately predicting trajectories for chaotic dynamical systems is a great scientific challenge of societal importance. For instance, despite the impressive progress of numerical weather prediction \citep{richardson1922weatherprediction, bauer2015quiet}, current methodologies still struggle to forecast extreme events. Heat waves \citep{perkins2013measurement}, flooding \citep{mosavi2018flood}, and oceanic rogue waves \citep{dysthe2008oceanic}, are examples of catastrophic events of enormous socio-economic impact. Part of the difficulty of forecasting such extreme events can be attributed to the chaotic behavior of the dynamical systems associated with weather prediction \citep{lorenz1963deterministic, hochman2019weatterdynsystem}. 

Furthermore, extreme events are often located at the tail of the distribution and are non-trivial to characterize, rendering them hard to sample efficiently, which in return has spun several methods seeking to attenuate this issue ~\citep{Kahn1951,rosenbluth1955,farazmand2019extreme,Qi2020mlextremeevents}.  Most of the methods above are based on rejection-sampling~\citep{hastings1970mcmc,Rossky1978langevinmcmc}, whose cost, measured as the number of on-demand large-scale simulations of the dynamical system, increases as events become rarer, rapidly becoming prohibitive.

Recent advances in deep generative models, particularly diffusion-based models, have shown remarkable results in capturing statistics of high-dimensional variables (such as images) and generating new samples from the learned probabilistic models~\citep{sohl2015deep,ho2020denoising,song2019generative,song2020score}. In constrast with GANs \citep{goodfellow2020generative} which often struggle with dropping modes that are difficult to model, diffusion models have proven better at capturing the full diversity of the data. 

In this paper, we investigate 
the application of diffusion models to modeling trajectories.
In particular, we are interested in whether such models can be used as surrogate models for the physical systems. We consider three classical dynamical systems: Lorenz strange attractors (``butterfly'') (\autoref{fig:front} left), the double pendulum, and FitzHugh-Nagumo neuron model  (\autoref{fig:front} right). Even though these models are fairly simple they retain the core difficulty of more complex ones, e.g., the first two exhibit chaotic trajectories while the last one exhibits extreme events.

Concretely, we study two questions. First, \emph{can we learn diffusion-based generative models directly from trajectories without explicit knowledge of the underlying differential equations and sample high-fidelity trajectories from the models?} Second, \emph{can we condition the sampling process to generate user-specified events --- trajectories of certain properties --- without the need to retrain the model?} Positive answers will enable researchers and practitioners to query the learned models with amortized computational costs, and the flexibility of studying events in tails of the distribution.

We answer the first question by building diffusion models capable of learning the trajectories of the three classical systems mentioned above. The models can produce trajectories with low error and calibrated uncertainties even when the underlying system is chaotic.

We answer the second question by deriving an approximation scheme to compute the conditional score function that enables conditioning on user-defined nonlinear statistics at inference time. The key idea is to use moment-matching to derive the distribution of the denoised sample conditioned on a noised sample, which is typically intractable. We show that the approximation becomes exact as the noise scale
vanishes.
Using this method, we can directly sample events in the tail of the distribution
and quantify their likelihood.

\begin{figure}[!t]
\centering
\begin{tabular}{cc}
    \includegraphics[width=0.22\textwidth]{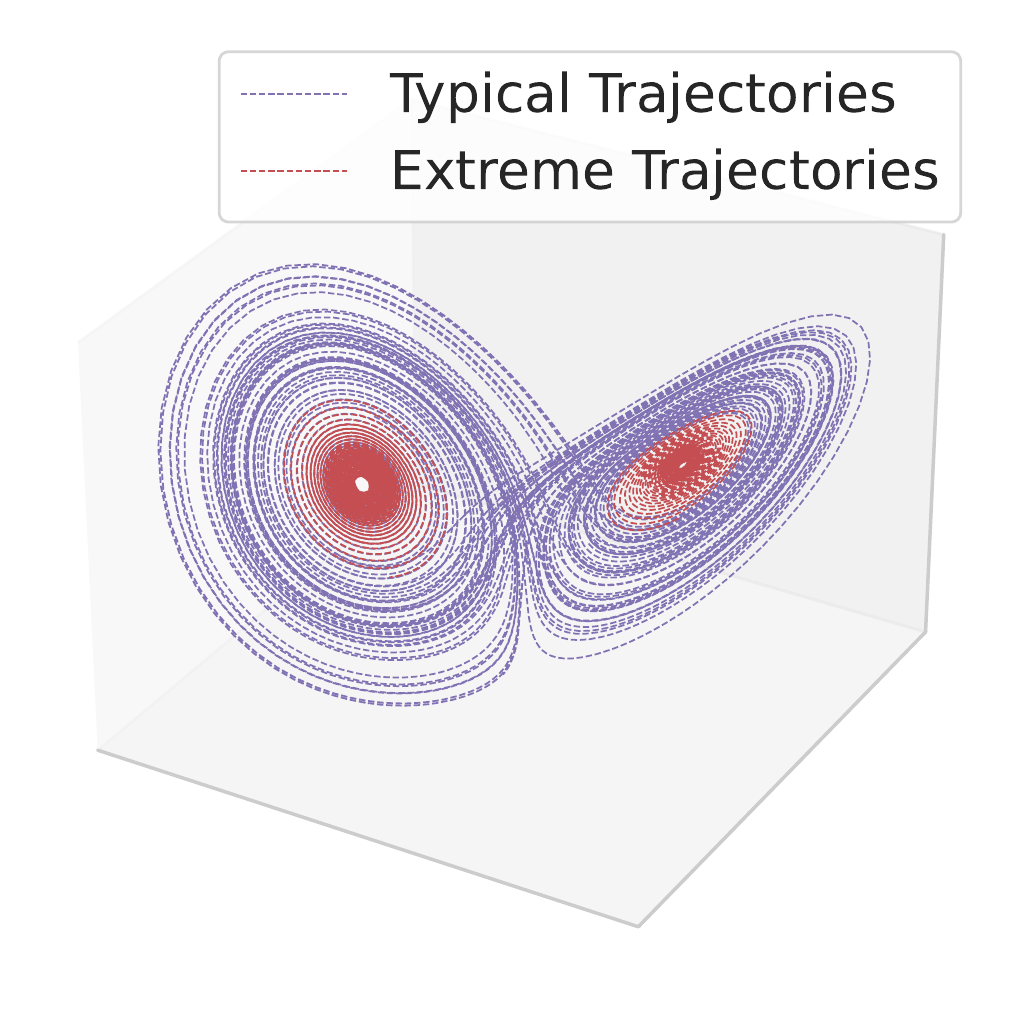} &  
    \includegraphics[width=0.22\textwidth]{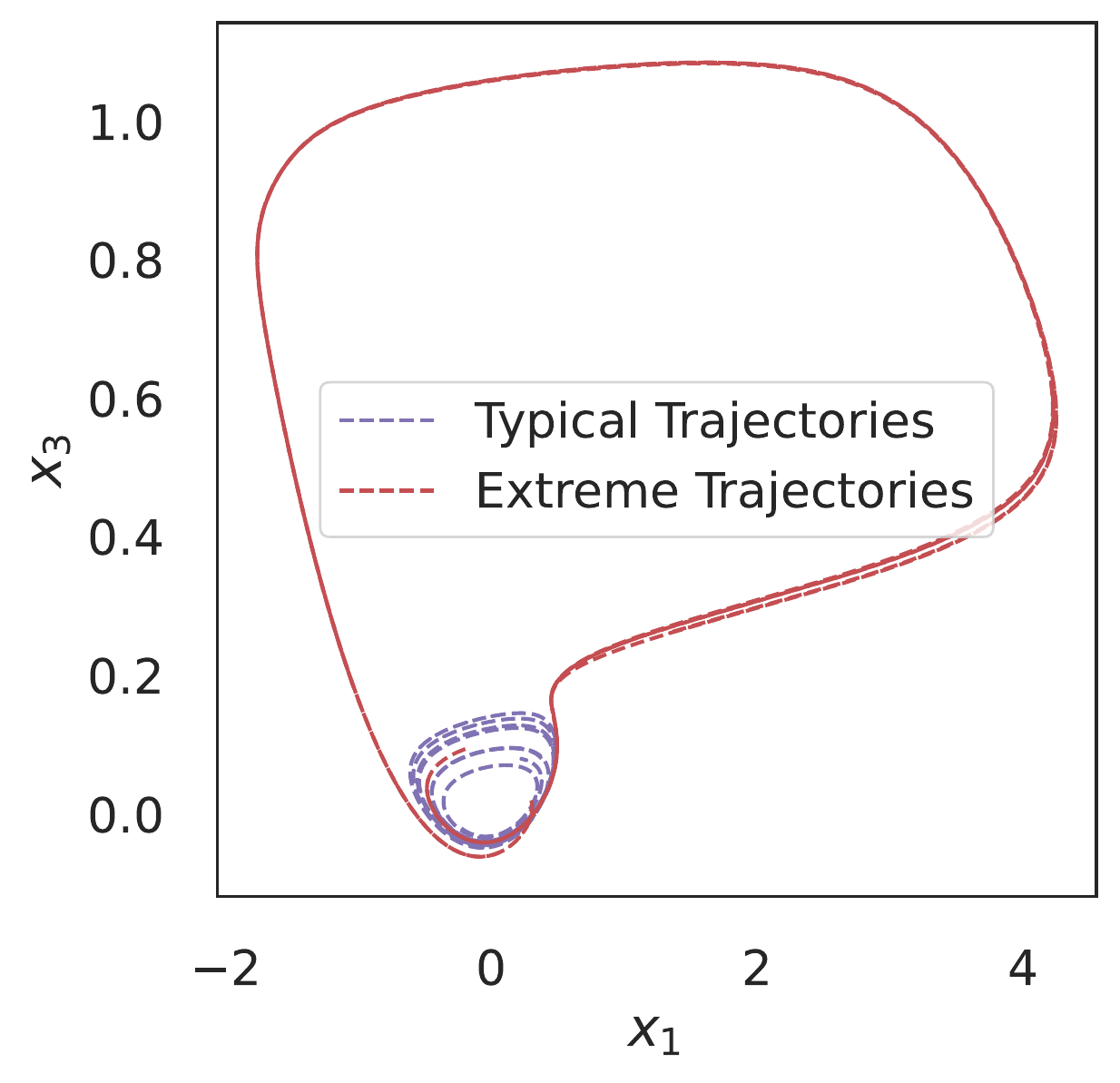} \\
    \end{tabular}
    
\caption{Chaotic nonlinear dynamical systems often have outlier events, and these events can be difficult to predict due to the chaotic nature of the system. \textbf{Left:} Trajectories of the Lorenz attractor, split into the trajectories which do not cross over to the opposite arm of the attractor in a given time horizon vs those that do. \textbf{Right:} Trajectories of the FitzHugh-Nagumo model, which feature the rare and unpredictable neuron spikes shown in red, which are nestled in with typical trajectories shown in purple.}
\label{fig:front}
\end{figure}
\section{Related Work}

\textbf{Denoising Diffusion Probabilistic Models}
Denoising Diffusion Probablisitic Models (DDPMs) \citep{sohl2015deep,ho2020denoising,song2019generative} construct a forward process where each training example from the data distribution is sequentially corrupted by increasingly larger noise. At the final step of this process, the sample is distributed according to a standard Gaussian distribution, completely erasing the data. The reverse process defines a generative model where, starting from a standard Gaussian sample, we follow the reverse denoising process using a neural network. These diffusion models are trained using score matching \citep{hyvarinen2005estimation, vincent2011connection} or denoising \citep{ho2020denoising} objectives. \citet{song2020score} introduce a continuous formulation of the diffusion process using stochastic differential equations (SDEs). Furthermore, by leveraging the connection with Neural Ordinary Differential Equations (NeuralODEs) \citep{chen2018neural}, \citet{song2020score} show how to perform exact likelihood computation.

\textbf{A posteriori conditioning}
Inference time (\textit{a posteriori}) conditioning is a promising and powerful paradigm for training large prior models and using them to perform different downstream tasks like inpainting, colorization, reconstruction, and solving general inverse problems. In contrast with (\textit{a priori}) train time conditioning, where the form of the conditioning must be known and used at training time, inference time conditioning enables using a unconditional diffusion model as a \emph{prior} and then conditioning it on different observations at inference time.

\citet{song2020score} provide a crude approximation of the conditional score function to perform inpainting. \citet{meng2021sdedit} use a similar method, but instead they perform unconditional generation from a noised version of a guide image. \citet{chung2022score} and \citet{chung2022come} generalize the method to linear transformations (such as in MRI reconstruction) and apply an explicit projection onto the constraints at each iteration. As shown by \citet{lugmayr2022repaint} and others \citep{chung2022improving}, these projection strategies (e.g. replacing noised versions of known pixels with the Gaussian samples $p(x_t|x_0)$ for inpainting) produces samples that lack global coherence: inpainted regions do not properly integrate information from the known regions. \citet{lugmayr2022repaint} address this by iterating forwards and backwards multiple times in order to better harmonize the information. \citet{chung2022improving} take a different approach with manifold constrained gradients (MCG): they combine constraint projection with an additional term that encourages the conditional sample iterates $x_t$ to lie on the data manifold given by $p(x_t)$. Both methods address the global coherency issue of prior methods; however, they are only valid for linear equality constraints and it is unclear how the samples relate to the true conditional distribution of the generative model.

\citet{graikos2022diffusion} take a more general approach for inverse problems using optimization, but only produce point estimates. Recently, \citet{chung2022diffusion} proposed an improved version of MCG that removes the projections and enables it to work both when there is measurement noise and nonlinear constraints. This method makes a probabilistic approximation of the score function, which becomes degenerate as the measurement noise goes to $0$, thus limiting its applicability for tail sampling of the deterministic events we consider here.

\textbf{Likelihoods} Even though diffusion models are not explicitly designed as likelihood models, they can be used to compute the likelihood $p(x_0)$ of a particular data point according to the model.
In order to obtain more than a lower bound on the likelihood, one needs to exploit the connection between the probability flow ODE and continuous normalizing flows \citep{grathwohl2018ffjord} as done in \citet{song2020score}, unfortunately, such connection is only available in continuous time.

\textbf{Extreme event prediction}
\citet{Qi2020mlextremeevents} apply deep convolutional neural networks to the prediction of extreme events in dynamical systems. \citet{wan2018data} use reduced order modelling in conjunction with LSTM-RNN networks to model extreme events. \citet{asch2022model} tackle the problem of lack of data when training deep networks for extreme event prediction.
\citet{doan2021short} use reservoir-computing based model to forecast extreme events.
\citet{guth2019machine} use machine learning to detect extreme events in advance from a given trajectory; rather than modeling trajectories of dynamical systems directly.
In contrast to the above works, our approach uses likelihood-based generative models which allows us to provide probability estimates for the extreme event occurring. Deep generative models, particularly NeuralODEs \citep{chen2018neural,yildiz2019ode2vae} have shown promise in modeling dynamical systems \citep{lai2021structural}, but their effectiveness in capturing tail events is yet to be ascertained.

\section{Diffusion Model for Dynamical Systems}\label{sec:background}

\paragraph{Background}  Diffusion models are composed of a forward noising process and its corresponding backward denoising process, which we describe below in its continuous-time formulation \citep{song2020score}.

The forward process evolves a given clean signal $x_0 \in \mathbb{R}^d$ from time $t=0$ to $t=1$ 
via the Itô Stochastic Differential Equation (SDE) 
\begin{equation*}
  dx_t  = f(x_t, t) dt +g(t) dW,  
\end{equation*}
where $W$ is the Wiener process, adding noise at each step.  Ultimately, $x_t$, the signal at time $t$, is a transformed (due to the drift term $f(x,t)$) and noised (due to the diffusion term $g(t) dW$) version of $x_0$. These values are also chosen so that the marginal distribution $p(x_1)$ is simply a spherical Gaussian. Here we denote the distribution of the data given the noise level at time $t$ as $p(x_t)$, though this distribution depends on $t$ and is sometimes written as $p_{X_t}(x_t)$ or $p_t(x_t)$.

If $f(x,t)$ is affine, the noise kernel $p(x_t |x_0)$ can be computed in closed-form. In this work, we define
\begin{equation}\label{eq:schedules}
f(x, t) =\tfrac{\dot{s}_t}{s_t} x \,\,\text{and}\,\, g(t)^2 = \sigma_t \dot\sigma_t - \sigma_t^2 \tfrac{\dot{s}_t}{s_t}.
\end{equation}
If these relations are specified, then the noise kernel has a simple form:
\begin{equation}
    p(x_t|x_0) = \mathcal{N}(x_t; s_tx_0,\sigma_t^2I).
\end{equation}
Namely, the diffusion model describes how $x_0$ is scaled and blurred. The derivation of \autoref{eq:schedules} can be found in \citet{sarkka2019applied} and specific choices of $s_t$ and $\sigma_t$ are described in \citet{song2020score,karras2022elucidating}.

The reverse process removes the noise from data. Specifically, given a noisy sample $x_1$ from the marginal distribution $p(x_1) = \mathcal{N}(0, \sigma_1^2I)$, the backward process evolves $x_1$ to $x_0$ by the following SDE
\begin{equation}\label{eq:sde}
dx_t = \big(f(x_t,t)-g(t)^2\nabla_{x_t} \log p(x_t)\big)dt +g(t)  d\bar{W},
\end{equation}
where $\bar{W}$ is the Wiener process running backwards.  $\nabla_{x_t} \log p(x_t)$ is the score function for the noised data, defined as the gradient of $\log p(x_t)$, the marginal probability of $x_t$. The diffusion model can be seen as a sequence of denoising steps according to the model score $s_\theta(x_t, t)$ aimed to match the noised score function of the true data $\nabla_{x_t} \log p_\mathrm{data}(x_t)$. This can be achieved with standard optimization algorithms on the score matching loss. From now on we will refer to $\nabla_{x_t} \log p(x_t)$ not as the noised scores of the data distribution, but as the noised scores of the model distribution.

We use the continuous time score matching formulation of diffusion models \citep{song2020score} to enable likelihood computations, and we train using the score matching loss
\begin{equation*}
    L(\theta) = \E_{(t,x_0,x_t)}\|s_\theta(x_t,t)-\nabla_{x_t}\log p(x_t|x_0)\|^2/\sigma_t^2,
\end{equation*}
where $x_0 \sim \mathcal{D}$ (the data distribution), 
$x_t \sim p(x_t|x_0)$, and $t\sim \mathrm{U}[0,1]$. For hyperparameters and additional training setup, see \autoref{app:hyperparams}. After training, we sample from the model using the Euler-Maruyama integrator with $1000$ uniformly spaced timesteps applied to the SDE in \autoref{eq:sde}.

\paragraph{Application to Dynamical Systems}
The trajectory of a dynamical system for a given initial condition $x(0)$, is the function  $x(\tau): [0, T] \rightarrow \mathbb{R}^d $, which is the solution to 
\begin{equation*}
    \dot{x}(\tau) = g(x(\tau),\tau),
\end{equation*}
for a given dynamics function $g$. We use ODE time $\tau$ to distinguish it from the diffusion time $t$. Note that the initial condition $x(0)$ should not be confused with $x_0$, where the latter refers to a noise free data point. 

We suppose that the initial condition $x(0)$ follows a certain distribution, which in return generates a distribution of trajectories. We assume each trajectory is discretized into $m$ timesteps, yielding a $m\times d$-dimensional array. We use diffusion models to model the collection of the trajectories. The architecture of the diffusion model's score function is described in \autoref{app:architecture}.

\section{A Posteriori Conditioning}
\label{sMethod}

Once $p(x)$ is learnt, we would like to obtain samples from it with properties of interest. Abstractly, we seek
\begin{equation*}
    p( x_0| {E})
\end{equation*}
where the property ${E}$ is a set given by ${E} = \{x_0: C(x_0) = y\}$, for equality constraints or ${E} = \{x_0: C(x_0) \le y\}$, for inequality constraints, for some smooth function $C:\mathbb{R}^{md} \rightarrow \mathbb{R}^n$.
This construction is fairly general and it can be used for different downstream applications, e.g., in the task of image inpainting, the property $C(x_0)=y$ encodes whether the observed portion of a sampled image $x_0$ corresponds to known pixel values $y$.

In order to perform \emph{a posteriori} conditional sampling using the learned model, we use the score function of the conditional distribution
\begin{align}
\nonumber
     \nabla_{x_t} \log p(x_t|{E}) & = \nabla_{x_t} \log p({E}|x_t) + \nabla_{x_t}\log p(x_t)  \\ 
     &= \nabla_{x_t} \log p({E}|x_t) + s_\theta(x_t, t) \label{ConstrainedScore} 
\end{align}
in the reverse diffusion process.
The challenge is to compute the first term while knowing the definition of ${E}$ only \emph{after} $p(x_t)$ is learned (without knowledge of ${E}$).

Directly computing this quantity is hard: $p({E}|x_t) =  \int \mathbb{I}_{[x_0\in E]}p(x_0|x_t) dx_0$, which is intractable.  A naive 
approach would be to use $p({E}|x_0)$ in place of $p({E}|x_t)$.
However, this approach leads to conditional samples lacking global coherence \citep{lugmayr2022repaint,chung2022improving}.
Instead, we derive an approximation to  $p(x_0|x_t)$ based on moment matching, which we use to perform conditioning with linear and nonlinear equality constraints as well as inequality constraints.

\subsection{Moment-matching Based Approximation}

We can view the forward diffusion process as a Bayesian inference task:
\begin{equation}
    x_0 \sim p(x_0),\quad p(x_t|x_0) = \mathcal{N}(x_t; s_t x_0, \sigma_t^2 I).
\end{equation}
To compute the mean of $p(x_0|x_t)$, we apply Tweedie's formula~\citep{robbins1992empirical,efron2011tweedie} 
\begin{align}
    \E[x_0|x_t] & = \frac{x_t+\sigma_t^2 \nabla_{x_t} \log p(x_t)}{s_t} := \hat{x}_0(x_t),
\end{align}
where $\hat{x}_0$ represents the best guess for $x_0$ given a value of $x_t$.  
Furthermore, Tweedie's formula can also be applied to the higher moments.
As we derive in \autoref{app:tweedies}, the conditional covariance matrix can be expressed exactly as
\begin{equation}
    \mathrm{Cov}[x_0|x_t]=\bigg[\frac{\sigma_t^2}{s_t^2}\big(I+\sigma_t^2 \nabla^2 \log p(x_t)\big)\bigg] := \hat{\Sigma}(x_t),
\end{equation}
where $\nabla^2 \log p(x_t)$ is the Hessian of the log probability, or, equivalently, the Jacobian of the score function.

Using these two expressions for the conditional mean and covariance of $x_0$, we can approximate $p(x_0|x_t)$ with a Gaussian
\begin{equation}
   p(x_0|x_t)  \approx \mathcal{N}(\hat{x}_0,\hat{\Sigma}),
\end{equation}
which can be conveniently applied to constraining $x_0$ to satisfy the desired property ${E}$.

\subsection{Linear Equality Constraints}

As an example, we consider the goal of imposing the set of linear constraints ${E}=\{x_0: Cx_0 = y\}$ onto samples from the diffusion model for a given constraint matrix $C \in \mathbb{R}^{r\times d}$ and $y\in \mathbb{R}^r$. 
\begin{figure*}[ht!]
    \centering
    \includegraphics[width=\textwidth]{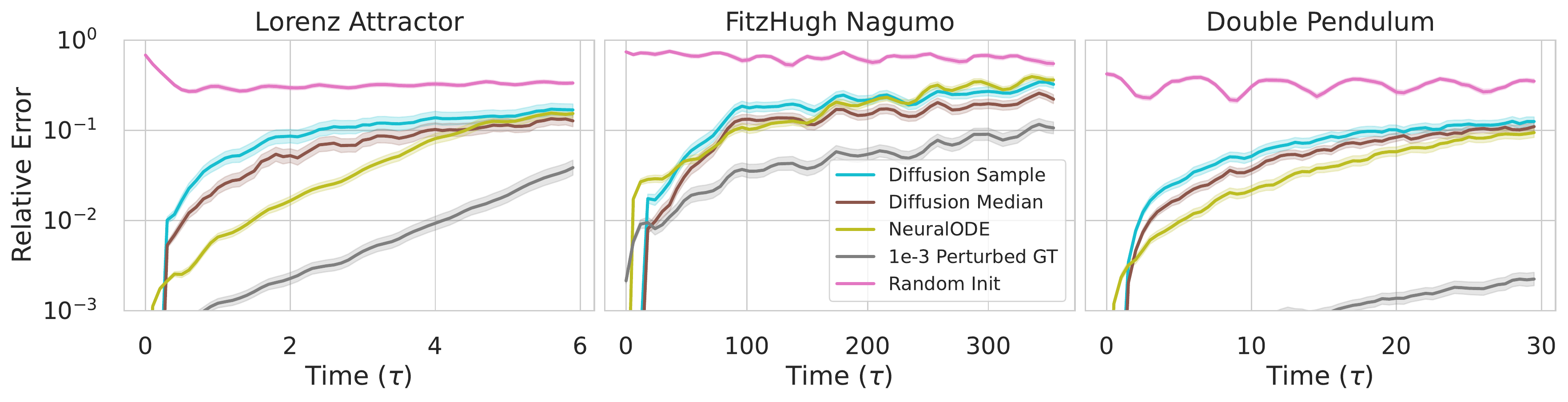}
    \caption{Relative error of sample diffusion model samples and pointwise median of $20$ samples when conditioned on initial condition compared to ground truth trajectories, NeuralODE rollouts, as well as errors for small perturbations of the initial condition evolved with the simulator, and independently sampled initial conditions for comparison. Shaded lines show two standard errors computed in log space. Diffusion model is evaluated on \textbf{Left:} Lorenz attractor, \textbf{Middle:} Fitzhugh-Nagumo model, and \textbf{Right:} double pendulum. The diffusion model and NeuralODE perform similarly, despite the diffusion model lacking ODE specific inductive biases of NeuralODE.
    }
    \label{fig:rollout_error}
\end{figure*}
The linear transformation $Cx_0$ leads to yet another Gaussian
\begin{equation}
    p(Cx_0|x_t) \approx \mathcal{N}(C\hat{x}_0,C\hat{\Sigma}C^\top).
\end{equation}
Both $\hat{x}_0$ and $\hat{\Sigma}$ depend on $x_t$, and the matrix $C\hat{\Sigma}C^\top$ can be computed using automatic differentiation involving the Jacobian of the score function. Specifically, the matrix $C\hat{\Sigma}$ can be computed as the Jacobian of the map $x_t \mapsto (\sigma_t^2/s_t)C\hat{x}_0$.  We now have,
\begin{equation}
    \nabla_{x_t} \log p({E}|x_t) \approx  \nabla_{x_t}\log \mathcal{N}(y; C\hat{x}_0,C\hat{\Sigma}C^\top )
    \label{eLinearConstraint}
\end{equation}
This enables using the modified score function ~\ref{ConstrainedScore} to sample from the constrained generative process.

\emph{Remark} In the case of linear constraints studied here, our work generalizes the recent work on constraining samples conditioned on linear constraints to be consistent with the data manifold~\citep{chung2022improving}. Specifically, if we approximate the Hessian as $\sigma_t^2 \nabla^2 \log p(x_t) \approx \lambda I$, we arrive at
\begin{equation*}
   \nabla_{x_t}\log p(x_t|y)  = s_\theta(x_t,t) - \frac{s_t^2}{2(1+\lambda)\sigma_t^2} \nabla_{x_t}\|C\hat{x}_0-y\|^2,
\end{equation*}
which reproduces eq.(14) in~\citep{chung2022improving}, ignoring constraint projections, and up to a scaling matrix $W$ 
lacking the ${s_t^2}/{\sigma_t^2}$ factor. In \autoref{app:linear_constraints_projection}, we show how the additional constraint projection for linear constraints arises under different circumstances with our moment matching approximation. 

However, using \citep{chung2022improving} or \citep{chung2022diffusion} directly with adaptive step size integrators in the continuous time formulation leads to numerical issues due to the stiffness of the problem; the ODE integrator step sizes shrink to zero and the integration does not complete. We explore these issues further in \autoref{app:convergence}. Observing that the ratio $s_t^2/\sigma_t^2 = \mathrm{SNR}$ (signal to noise ratio) typically varies over $12$ orders of magnitude in the region $t \in (0,1]$, we can now understand how the misscaling of this term leads to numerical problems. Adding the additional scaling factor $s_t^2/\sigma_t^2$ enables us to use the method in continuous time. Additionally, we can now leverage our probabilistic interpretation and the full covariance matrix to determine how to condition on nonlinear and inequality constraints.

\subsection{Nonlinear Equality Constraints}\label{subsec:nonlinear}

For a set of nonlinear constraints $C(x_0)=y$, we approximate $C(x_0)$ with its first-order Taylor expansion:
\begin{equation}
C(x_0) \approx  C(\hat{x}_0) + \nabla C\cdot (x_0 - \hat{x}_0).
\end{equation}
The Jacobian $\nabla C$ is evaluated at $\hat{x}_0$. In this ``linearized'' constraint, we approximate the desired probability using the result from the previous section on linear constraints:
\begin{equation}
 p(C(x_0)=y|x_t) \approx \mathcal{N}(C(\hat{x}_0), \nabla C\, \hat{\Sigma}\,\nabla C^\top).
 \label{eNonlinearConstraint}
\end{equation}

The quality of the approximation depends on how much the function $C$ varies over the scale of the variance of $x_0$, which we evaluate in \autoref{subsec:gaussian_approx}.

\subsection{Inequality Constraints}\label{sec:discrete_conditioning}

With the Gaussian approximations of \autoref{eLinearConstraint} and \autoref{eNonlinearConstraint} for linear and nonlinear equality constraints, we can also handle inequality constraints such as $E = \{x_0:  C(x_0) > y\}$. For a one-dimensional inequality constraint, 
\begin{equation*}
 p(E|x_t) = p(C(x_0)>y|x_t) \approx \Phi\bigg(\frac{C(\hat{x}_0)-y}{\sqrt{\nabla C^\top \hat{\Sigma}\nabla C}}\bigg),
\end{equation*}
where $\Phi$ is the Gaussian CDF.

Plugging the above into the score function for sampling conditional distribution \autoref{ConstrainedScore}, we can directly sample from the tails of the distribution according to any user defined nonlinear statistic $C(\cdot)$, focusing the model on extreme and rare events, using the conditional scores
\begin{equation*}
    \nabla_{x_t}\log p(x_t|E) \approx s_\theta(x_t,t) + \nabla_{x_t}\log \Phi\bigg(\frac{C(\hat{x}_0)-y}{\sqrt{\nabla C^\top \hat{\Sigma}\nabla C}}\bigg).
\end{equation*}
\eat{
When used in the context of understanding outliers and extreme events, we care not about the particular value that the sample has, but rather that it exceeds some threshold or lies in some interval.
Since we have this probabilistic approximation, we can employ it not just for equality constraints typical of inverse problems, but also for these inequality constraints. 

Consider the one dimensional inequality constraint $E = [C(x)>y]$ for some (possibly nonlinear) statistic $C$ exceeding some threshold (such as the temperature exceeds a certain threshold in a certain region for a weather model).
From our discussion on nonlinear constraints, $p(C(x_0)|x_t) \approx \mathcal{N}(C(\hat{x}_0),\nabla C^\top \hat{\Sigma}\nabla C)$. Therefore the constraint likelihood on the noised data can be well approximated as
\begin{equation*}
    p(E|x_t) = p(C(x_0)>y|x_t) = \Phi\bigg(\frac{C(\hat{x}_0)-y}{\sqrt{\nabla C^\top \hat{\Sigma}\nabla C}}\bigg),
\end{equation*} where $\Phi$ is the Gaussian CDF.

Using this observation likelihood the conditional score function for $p(x_t|C(x)>y)$ become
\begin{equation*}
\mywboxmath{
    \nabla_{x_t}\log p(x_t|E) = s_\theta(x_t,t) + \nabla_{x_t}\log \Phi\bigg(\frac{C(\hat{x}_0)-y}{\sqrt{\nabla C^\top \hat{\Sigma}\nabla C}}\bigg)}.
\end{equation*}

}

\begin{figure*}[t]
    \centering
    \begin{tabular}{ccc}
    \includegraphics[width=0.29\textwidth]{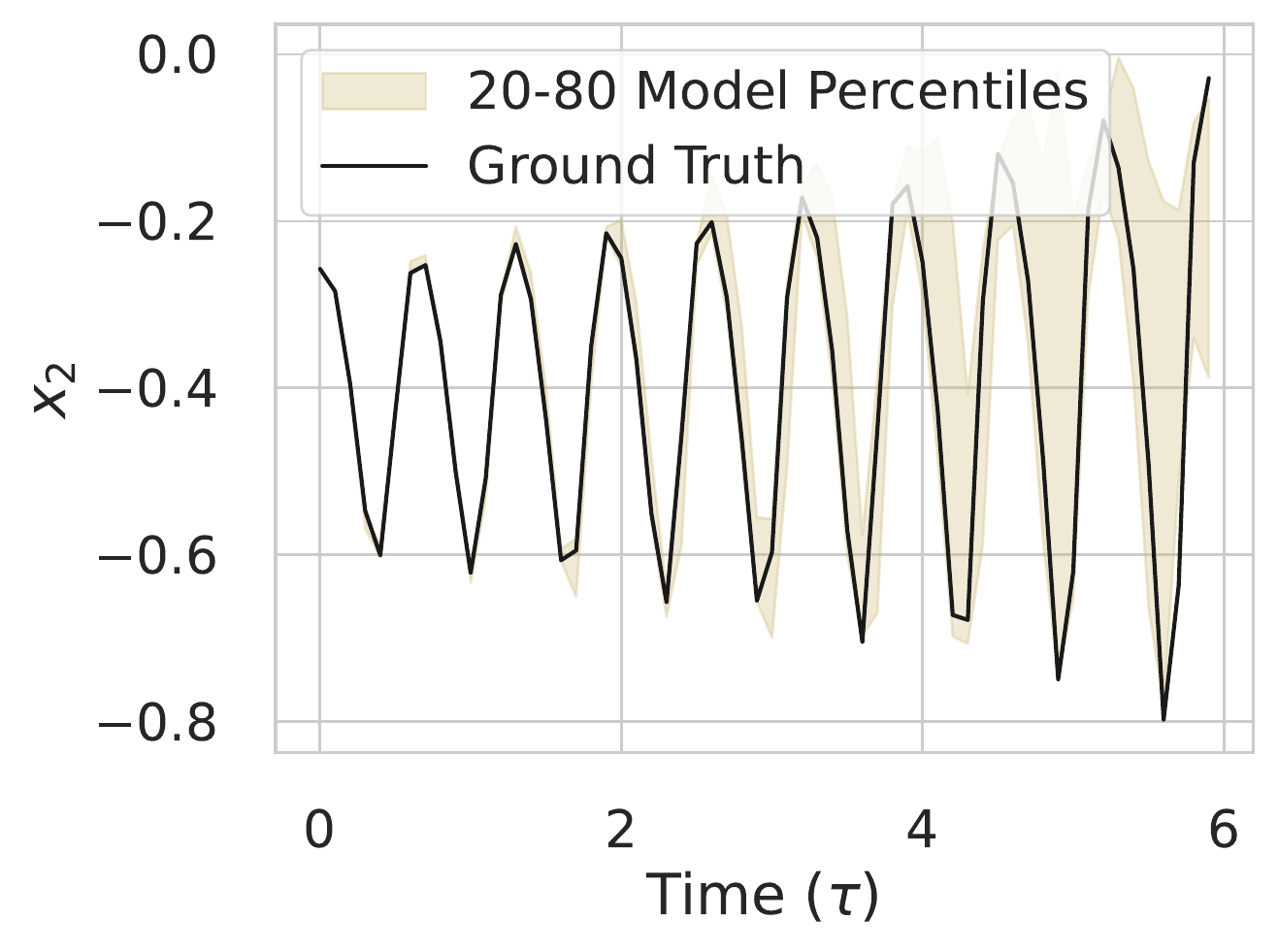} &  
    \includegraphics[width=0.29\textwidth]{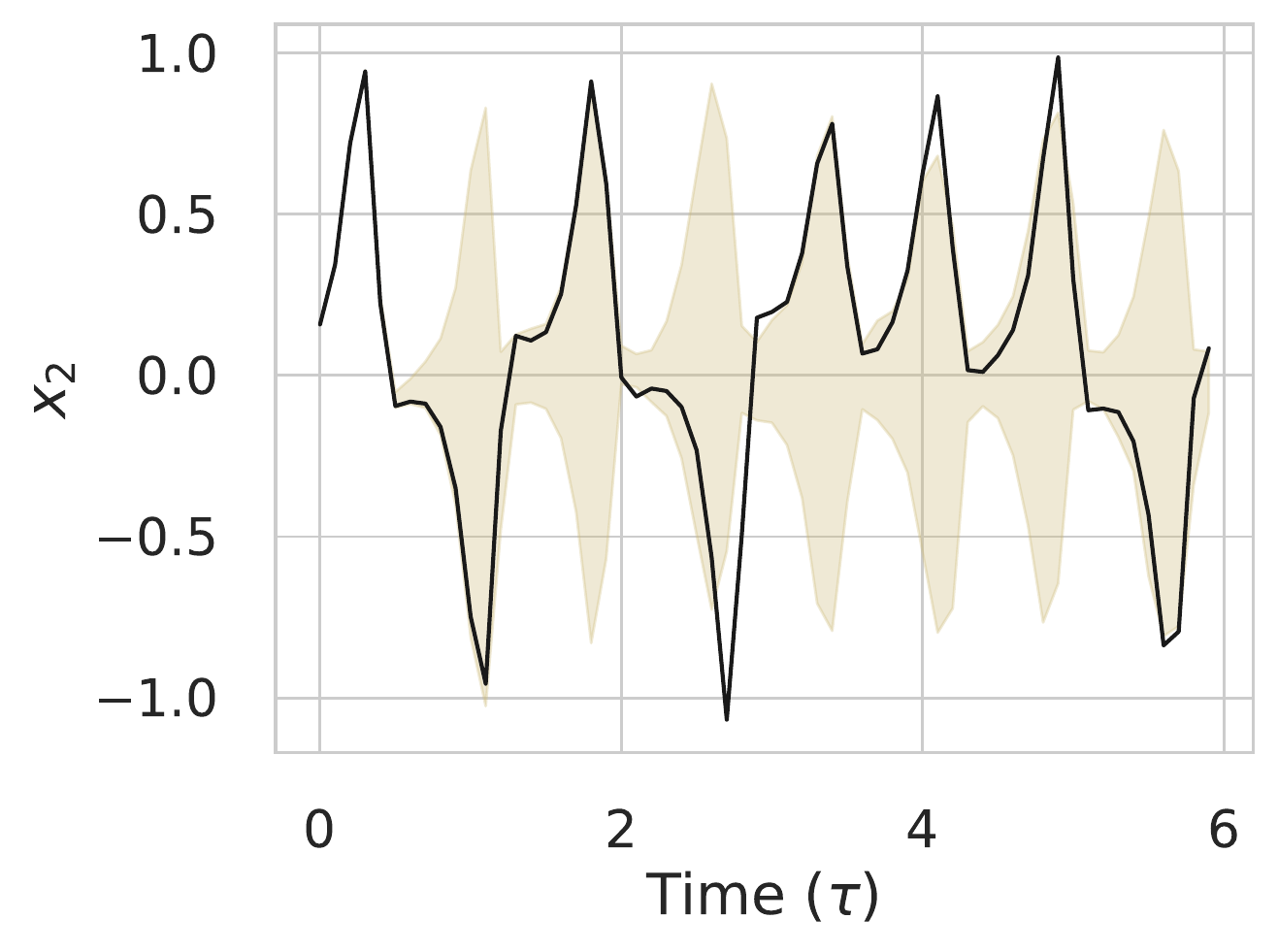} & 
    \includegraphics[width=0.37\textwidth]{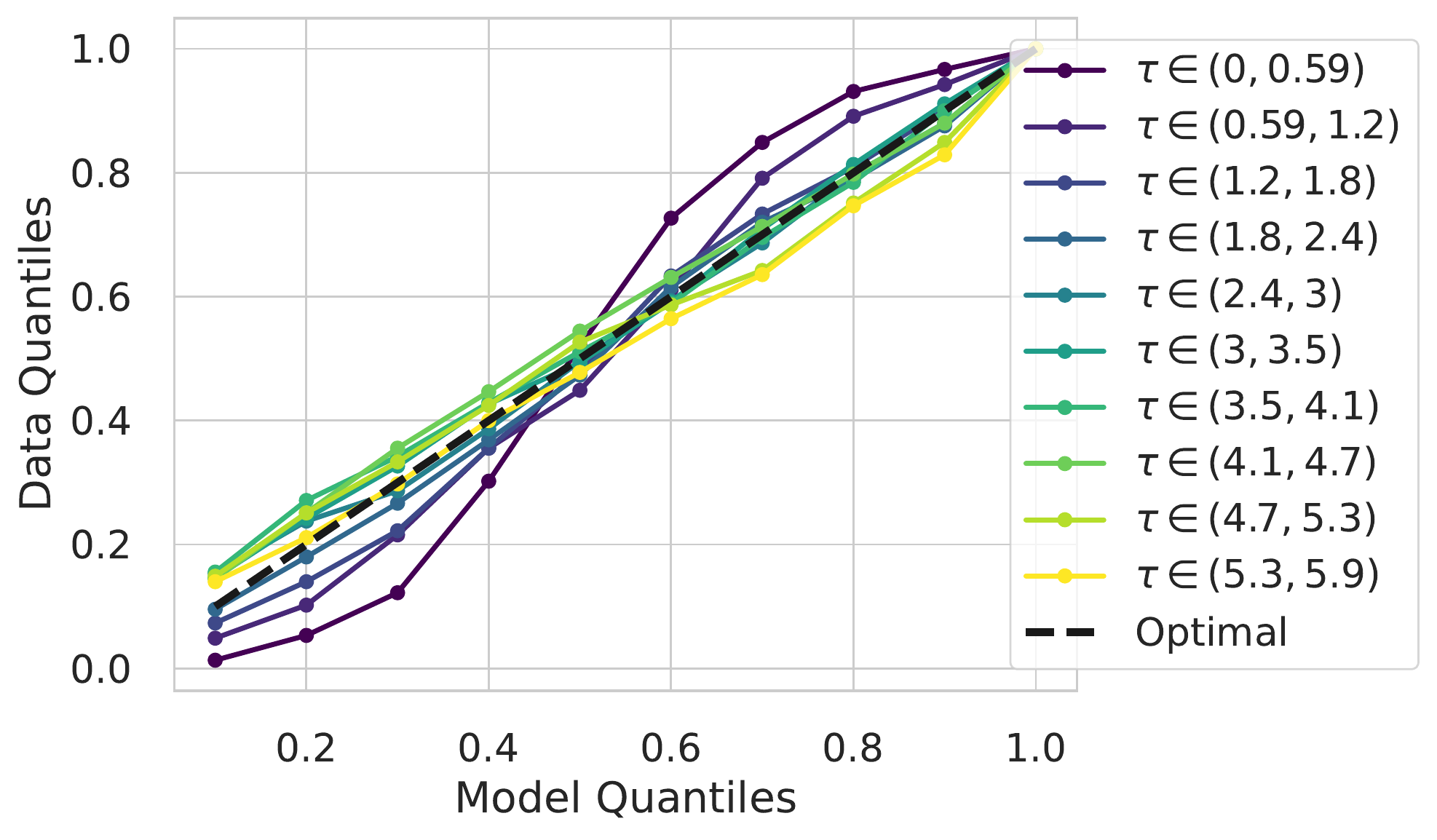} \\
    \end{tabular}
    \caption{Uncertainty quantification captured by the learned diffusion model. (\textbf{Left}) 20-80th precentiles state values for the $x_2$ component of the Lorenz attractor diffusion model, compared to the actual trajectory. As this trajectory lies in a region of the state space with a lower local Lyapunov exponent, the dynamics are less chaotic here, and the model uncertainties are relatively low. 
    (\textbf{Middle}) For samples in the very chaotic region, state transitions to the opposite arm of the attractor are common and captured by the uncertainties.
    (\textbf{Right}) Calibration of the predicted uncertainty quantiles for pointwise predictions as compared to the ground truth empirical quantiles, where optimal lies along $x=y$. While not perfect, the models are reasonably well calibrated, with bias to uncertainties being slightly to broad at early times and too narrow at late times.}
    \label{fig:uncertainty_quantification}
\end{figure*}

\subsection{Likelihoods of Events}\label{sec:event_likelihoods}

Since we can obtain samples from $p(x_0 |E)$, for a given user-defined event $E$, we can also compute the marginal likelihood of such an event.
Concretely, given a sample $x_0$, we can break down the likelihood using Bayes rule:
\begin{align}\label{eq:marginal}
\log p(E) = \log p(x_0) - \log  p(x_0 |E) + \log p(E|x_0).
\end{align}
When sampling $x_0 \sim p(x_0|E)$, the last term is zero when $x_0$ is a member of $E$. 
The two probabilities $p(x_0)$ and $p(x_0|E)$ can be computed by integrating the probability flow ODEs \citep{song2020score} 
\begin{align}\label{eq:prob_flow}
    \dot{x}_t &= f(x_t,t)-\tfrac{1}{2}g(t)^2 \nabla_{x_t} \log p(x_t),\\
    \dot{x}_t& = f(x_t,t)-\tfrac{1}{2}g(t)^2 \nabla_{x_t} \log p(x_t|E),
\end{align}
forwards in time with the continuous change of variables formula derived in FFJORD \citep{grathwohl2018ffjord}.
To reduce variance and the integration time, we use a second-order Heun integrator with a fixed time step, and compute the Jacobian log determinant directly using autograd. For more details, see Appendix~\ref{app:likelihood}.
While the above procedure is theoretically valid 
for a single sample $x_0$, we can average the estimate over multiple samples $x_0|E$ for improved accuracy.
\section{Results}

In order to evaluate the capability of diffusion models to express chaotic nonlinear dynamics, we train the models to fit the distribution of trajectories over different initial conditions for the test problems. For each system, we choose a Gaussian initial condition distribution, and integrate for a sufficient time to allow the distribution to reach equilibrium 
before recording the data. We choose $s_t$ and $\sigma_t$ according to the variance exploding process, \ie, $s_t=1$ and $\sigma_t = \sigma_\mathrm{min}\sqrt{(\sigma_\mathrm{max}/\sigma_\mathrm{min})^{2t}-1}$.
Equation parameters and collection details are specified in \autoref{app:datasets}. We train on the following three dynamical systems:

\textbf{Lorenz Attractor}
The Lorenz attractor \citep{lorenz1963deterministic} is a well-studied example of chaotic behavior, governed by a coupled three dimensional nonlinear ODE. The system contains two prominent arms of a strange attractor, and trajectories chaotically switch between the two arms.

\textbf{FitzHugh-Nagumo}
The FitzHugh-Nagumo model \citep{fitzhugh1961impulses} is a nonlinear ODE given by the coupled equations
\begin{align*}
    \tfrac{dx_i}{d\tau} &= x_i(a_i-x_i)(x_i-1) + y_i +k\sum_{j=1}^n A_{ij}(x_j-x_i), \\
    \tfrac{dy_i}{d\tau} & = b_i x_i - c_iy_i,
\end{align*}
 for $i=1,2$, leading to a set of four equations modeling the dynamics of two coupled neurons. The values for the parameters are specified in \autoref{app:datasets} and match the choice of \citet{farazmand2019extreme} which cause the system to exhibit rare but high magnitude neuron spikes, such as the examples shown in \autoref{fig:front} (right) and \autoref{fig:conditional_samples_fitzhugh} (d).

\textbf{Double Pendulum}
The double pendulum is another classic example of a chaotic system produced by the dynamics of a rigid pendulum with two point masses under the influence of gravity, but with trajectories having distinct energies unlike the previous two systems.
\begin{figure*}[t]
    \centering
    \begin{tabular}{cccc}
    \includegraphics[height=0.22\textwidth]{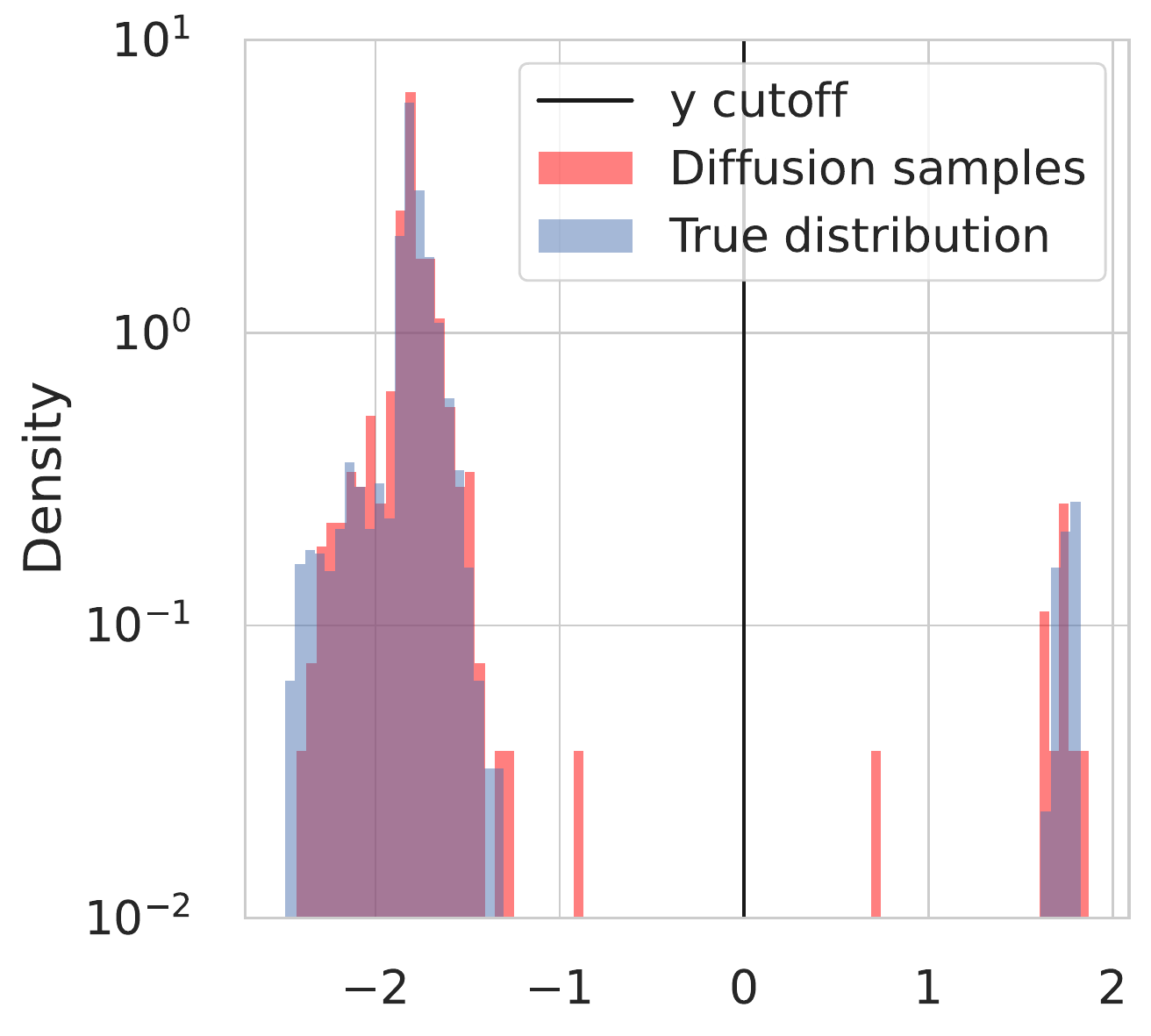} &  
    \includegraphics[height=0.22\textwidth]{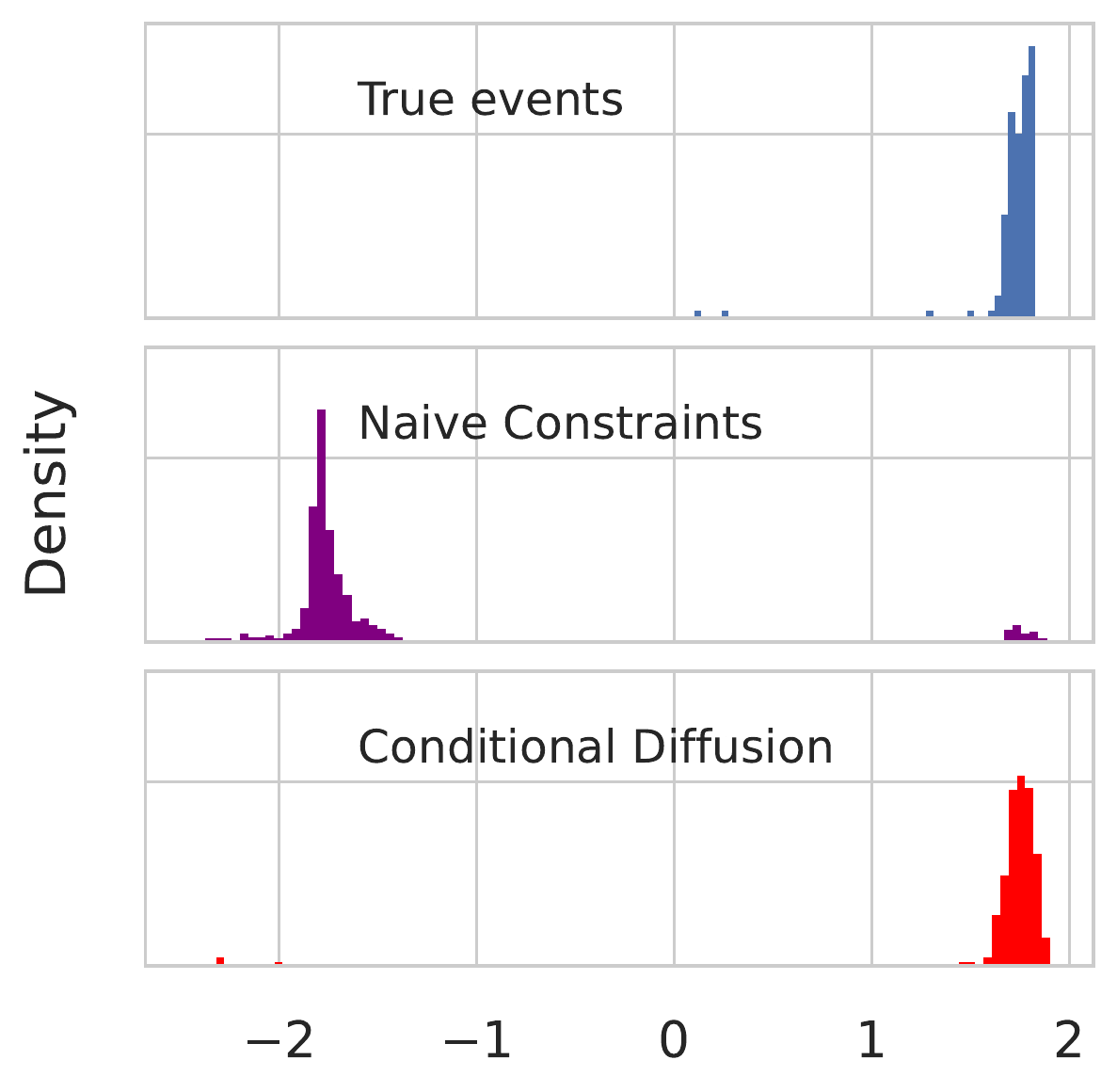} & 
    \includegraphics[height=0.225\textwidth]{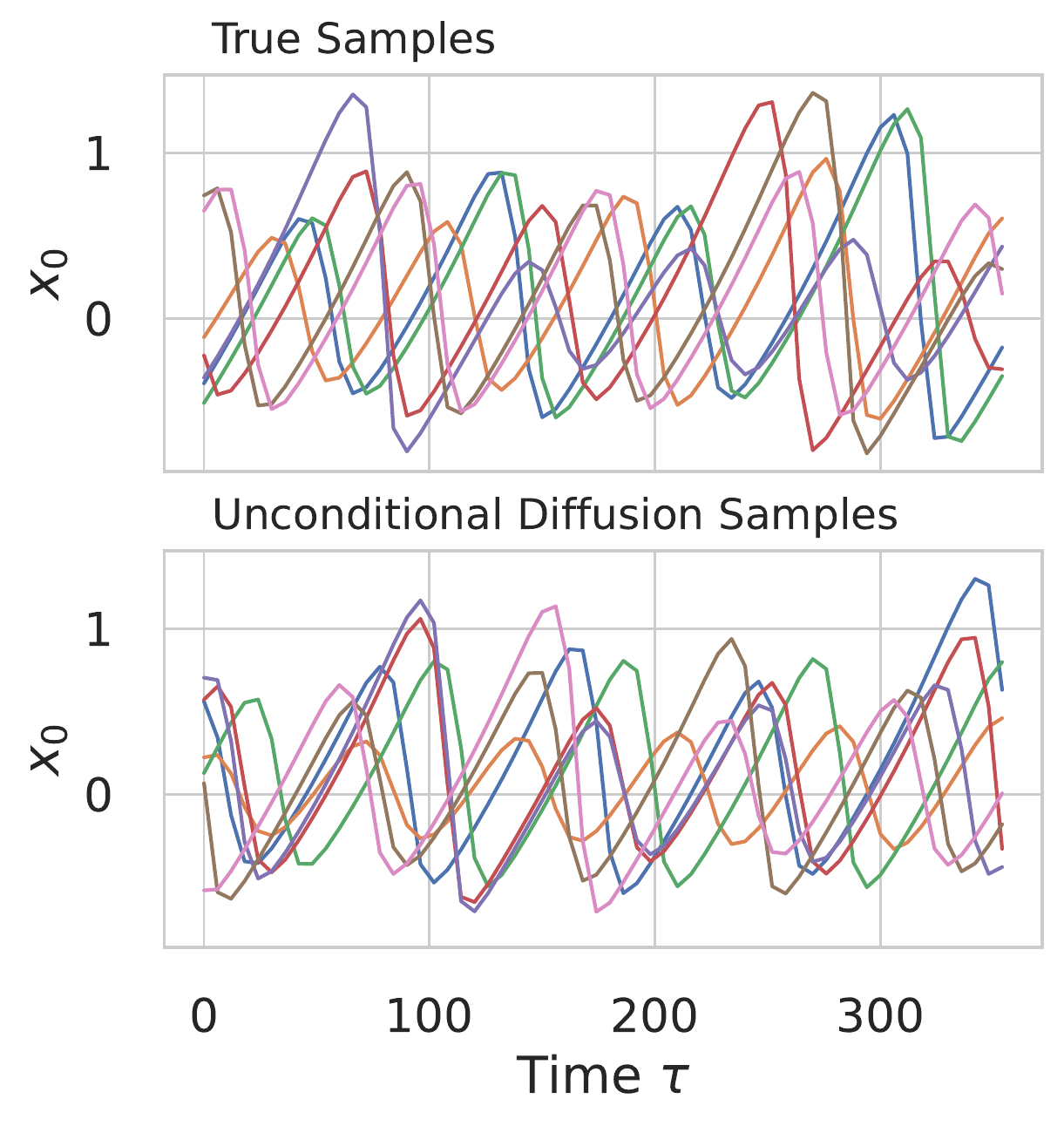} & 
    \includegraphics[height=0.225\textwidth]{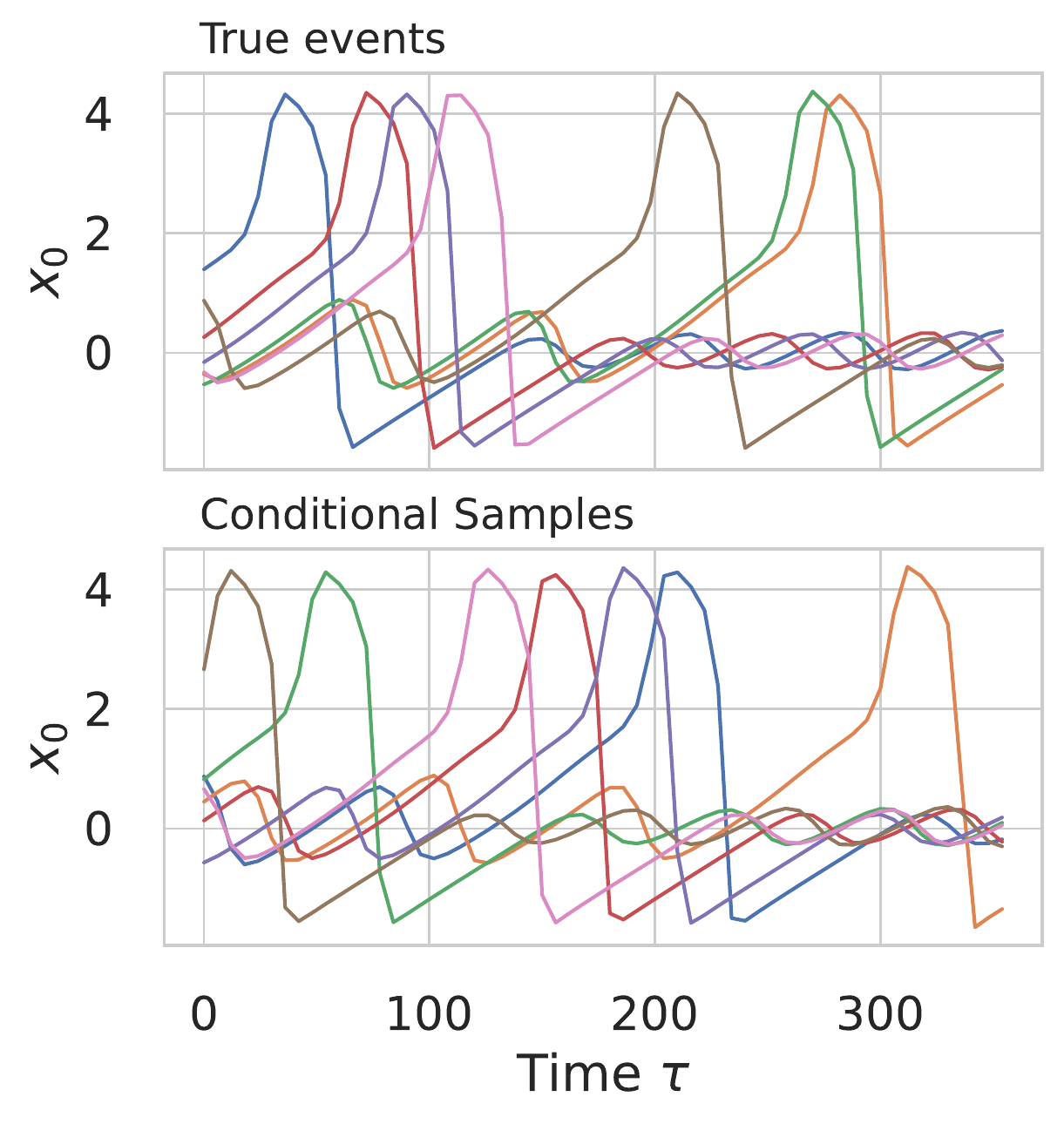}\\
    (a) $C(x)$& (b) $C(x)|E$& (c) $x \sim p(x)$ &(d) $x \sim p(x|E)$\\ 
    \end{tabular}
    \caption{(\textbf{a}) Histograms of statistic values $C(x)$ in the Fitzhugh-Nagumo system, for trajectories sampled unconditionally from from the dataset and the model. Here the event describes the unusual large scale neuron spikes that exist as a small cluster of outliers at $C(x)>1.5$. Notably the unconditional model distribution captures the outliers. (\textbf{b}) Statistic values on samples produced by our method of tail sampling vs actual extreme events vs naive approach to enforce the constraints. (\textbf{c}) Unconditional trajectory samples from both the data and model distributions.
    (\textbf{d}) Example event conditional samples compared to actual events. When conditioning on the event defined by the statistic value $E=[C(x)>0]$, the model is able to conditionally sample from the distribution, unlike for the naive approach to incorporate the constraints.
    }
    \label{fig:conditional_samples_fitzhugh}
\end{figure*}
\subsection{Unconditional Samples}

\textbf{Point Predictions}
We first evaluate the extent to which the model captures the system dynamics by measuring its ability to make accurate predictions of a trajectory given an initial condition. For each dataset, we train a conditional model that takes in the first three timesteps of a trajectory so as to make it conditional on the initial conditions (prior conditioning). Both the errors of a single sample and the pointwise median of $20$ samples are shown, which are computed against the ground truth from a held out set of initial conditions. We compare the prediction relative errors against NeuralODE models, which were trained on the same data (see \autoref{app:datasets} for details), and perturbed ground truth, in which we use a ground truth trajectory but with an initial condition perturbed by Gaussian noise with standard deviation $10^{-3}$. The last comparison specifies how the error of a small perturbation is chaotically amplified, and therefore provides a lower bound on performance. We measure relative error as $\mathrm{RelativeError}(a,b) = \|a-b\|/(\|a\|+\|b\|)$. \autoref{fig:rollout_error} shows that diffusion models capture the dynamics of the system performing similarly to NeuralODEs on pure point predictions, despite the fact that NeuralODEs make use of the ODE structure bias whereas the diffusion model does not.  

\textbf{Uncertainty Quantification}
While the model makes accurate point predictions, 
we are interested in the \emph{distribution} over the outputs captured by the model. 
In particular, we provide numerical evidence showing that the pointwise uncertainties over the state values are reasonably well calibrated, \ie, the quantiles of the model predictive distribution at a given point includes the truth the appropriate fraction of the time. 
In principle, there are two sources of uncertainty with the model when trained on these dynamical systems: uncertainty associated with the chaotic growth of error that introduces \textit{defacto} randomness, and the epistemic uncertainty of the model fit, both of which can be captured by the model. In \autoref{fig:uncertainty_quantification} (left, middle) we show the $20$-$80$th percentile of the state value of the 2nd component ($x_2$) of the Lorenz system per point as produced by our model in the shaded regions vs a ground truth trajectory, showing that the trajectory lies within the prediction interval and that it correctly captures that the state may venture to the other arm of the attractor (having larger $x_2$ value). In \autoref{fig:uncertainty_quantification} (right) we show that the model's uncertainties are calibrated quantitatively, by binning the empirical quantiles of the output samples per point and measuring the rate at which the actual values fall into those quantiles.

\subsection{Conditional Samples}
We showcase the ability to sample unlikely and extreme events with our method, by sampling extreme events in the Fitzhugh-Nagumo system described previously. Over the time-horizon considered, the neuron spiking events occurs in only in roughly $1/30$ of the  trajectories, which are qualitatively very different from the average trajectories. These unusual trajectories are shown in \autoref{fig:front} (right) and \autoref{fig:conditional_samples_fitzhugh} (d).
We define an event through the quantity
$C(x) = \max_\tau (x_1(\tau)+x_2(\tau))/2-2.5$, $E = [C(x)>0]$. The statistics of $C$ for the ground truth unconditional and conditional distributions are shown in \autoref{fig:conditional_samples_fitzhugh} (a) and (c) respectively. Note the unusual neuron fires out past value $4$, which are also produced occasionally when sampling the trained model. When conditioning on this nonlinear inequality constraint using our method, we can sample directly from this cluster of outliers producing event samples (\autoref{fig:conditional_samples_fitzhugh} b) that mirror true events, and match the statistic values (\autoref{fig:conditional_samples_fitzhugh} c). We compare our method (Conditional Diffusion in the figure) to a more naive approach (Naive Constraints) of using $p(E|x_t) \approx \Phi(\tfrac{s_t}{\sigma_t}(C(\hat{x}_0)-y))$ which doesn't make use of the 2nd order Tweedies formula information of the conditional covariance. Unlike our method, this approach is not as effective at sampling from the tail.

Next, we perform the same tail sampling for the Lorenz system. We separate trajectories in a region of the state space where switching between the two arms of the attractor is common vs where it is not. For this purpose we define the nonlinear statistic $C(x) = 0.6-\|F[x-\bar{x}]\|_1$ where $F$ is the Fourier transform applied to the trajectory time $\tau$ and $\|\|_1$ is the 1-norm taken over both the Fourier components and the $3$ dimensions of the state, and $\bar{x}$ is just the average value of $x$ over $\tau$. As shown in \autoref{fig:conditional_samples_lorenz}, this statistic separates the distribution into two populations. We condition on this inequality constraint $C(x)>0$ and we generate samples satisfying the inequality constraint and that do not change arms, as shown in \autoref{fig:conditional_samples_lorenz}. With both systems, the conditional samples preserve the diversity in the distribution, rather than collapsing to a conditional mode as one would have with optimization based methods.

\begin{figure*}[t]
    \centering
    \begin{tabular}{cccc}
    \includegraphics[height=0.22\textwidth]{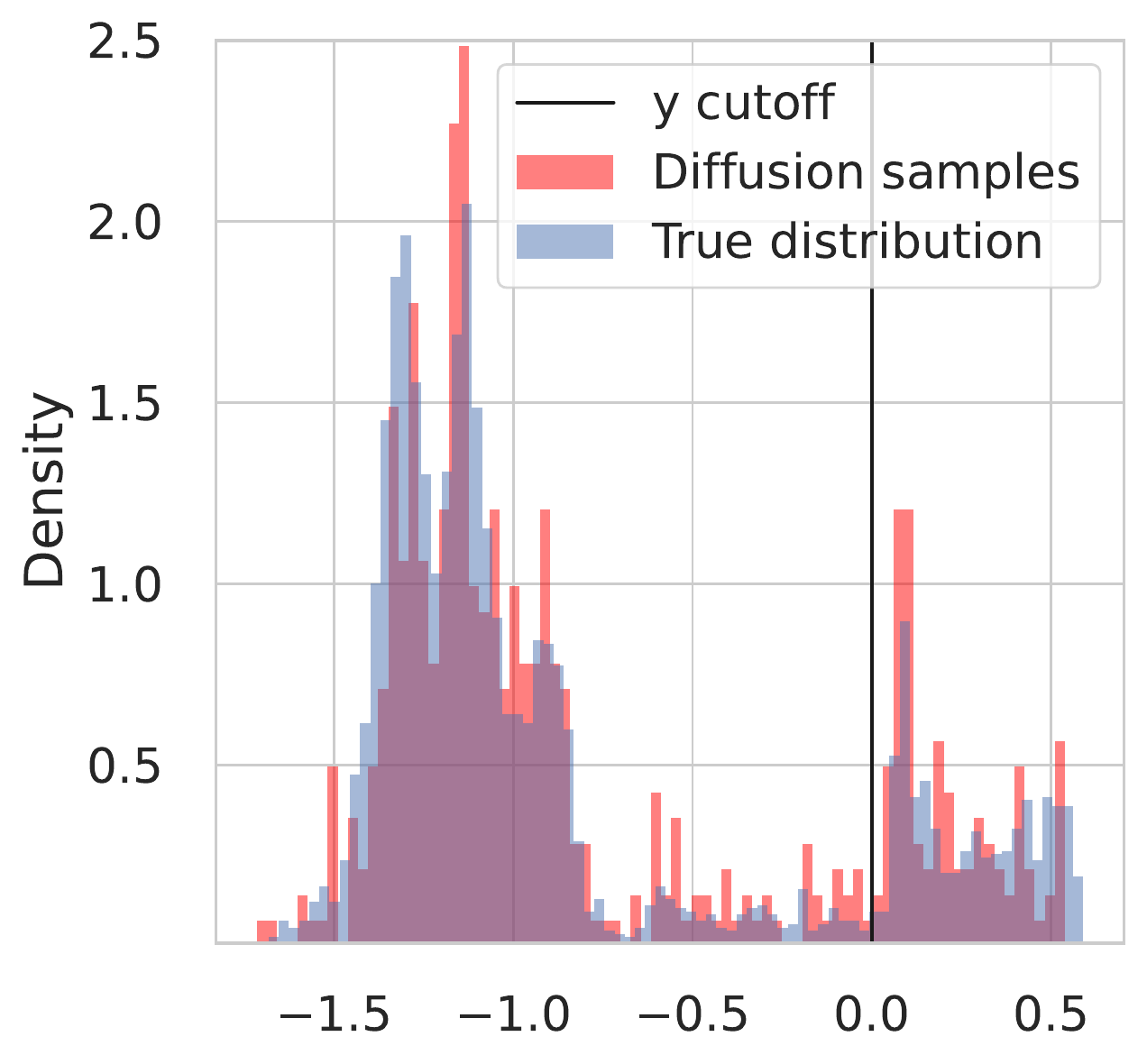} & 
    \includegraphics[height=0.22\textwidth]{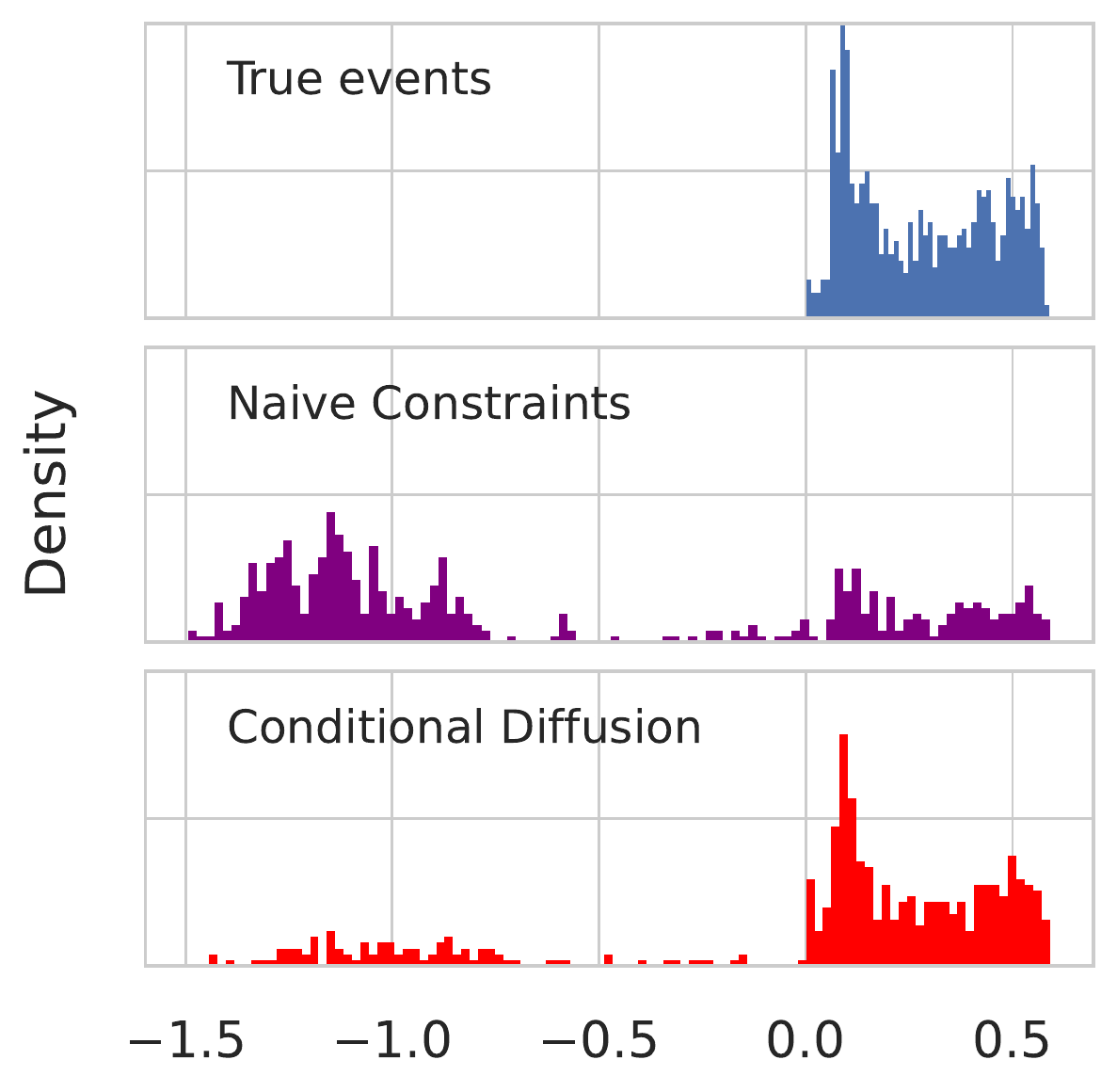} &
    \includegraphics[height=0.225\textwidth]{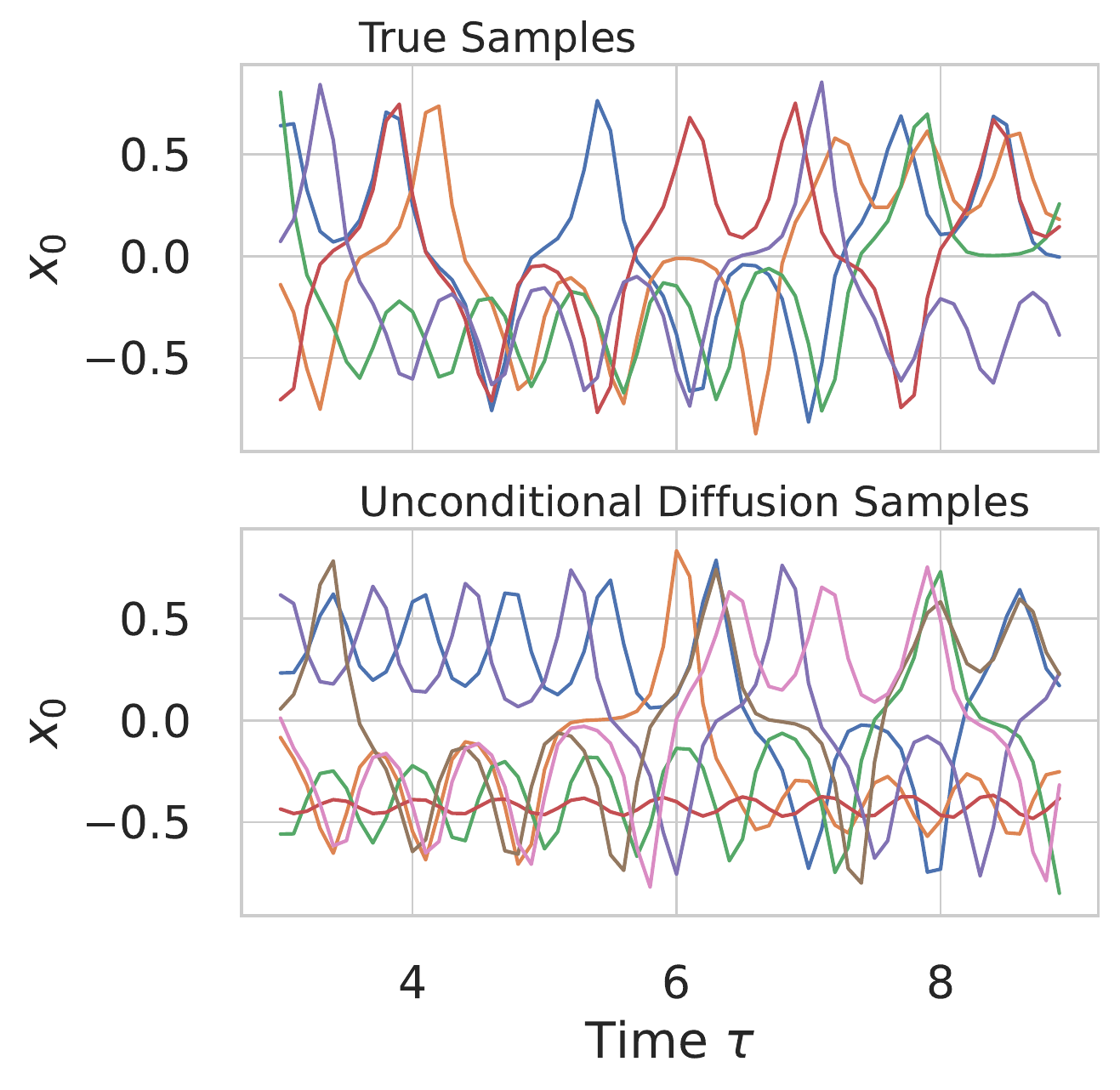} &
    \includegraphics[height=0.225\textwidth]{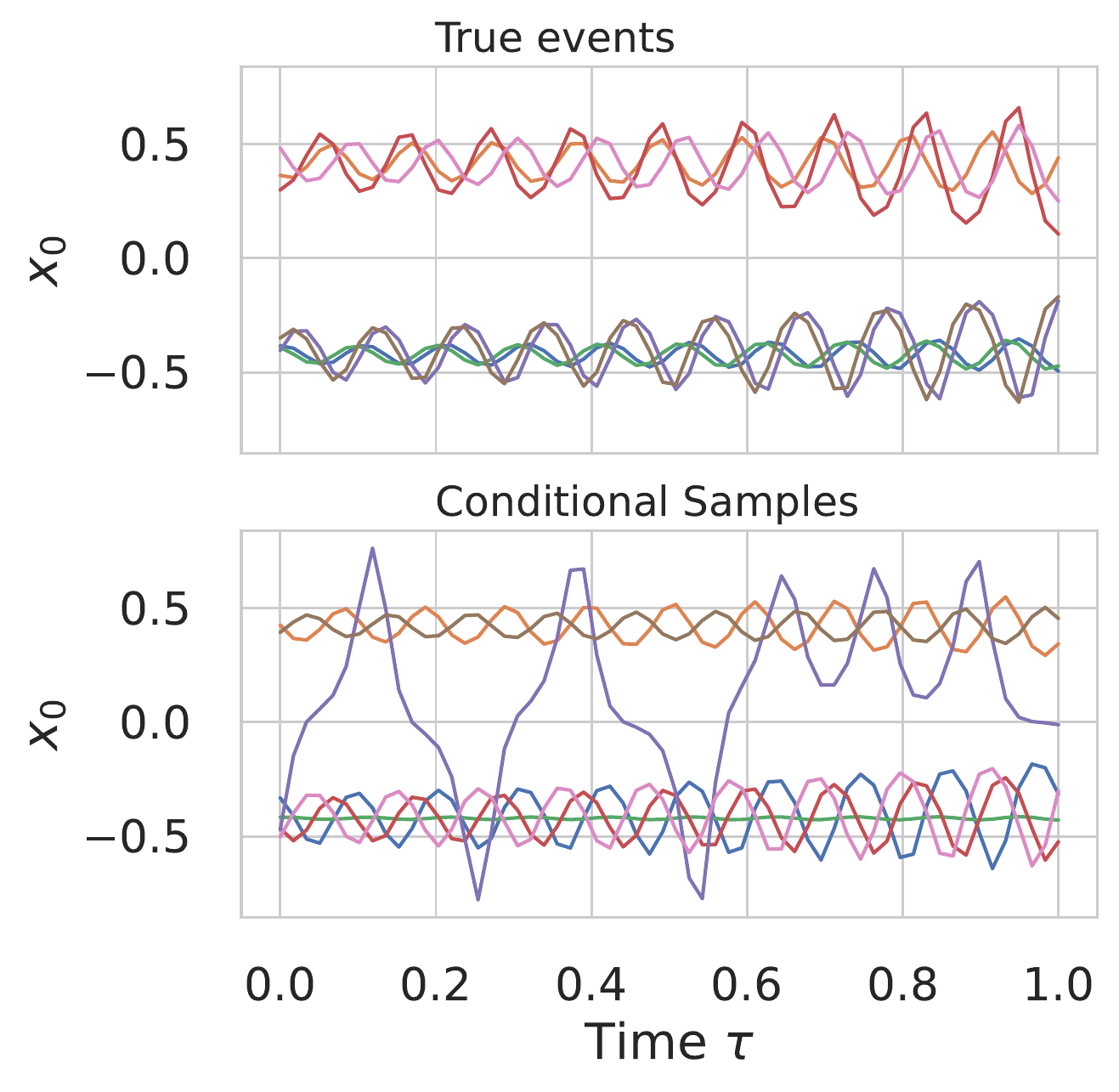} \\
    (a) $C(x)$& (b) $C(x)|E$& (c) $x \sim p(x)$ & (d) $x \sim p(x|E)$\\ 
    \end{tabular}
    \caption{(\textbf{a}) Histograms of statistic values $C(x)$ in the Lorenz system, for trajectories sampled unconditionally from from the dataset and the model. The event describes trajectories which do not cross between the two arms of the strange attractor (in the limited time horizon). Notably the unconditional model distribution captures the outliers. (\textbf{b}) Statistic values on samples produced by our method of tail sampling vs actual extreme events vs naive approach to enforce the constraints. (\textbf{c}) Unconditional trajectories from the data and model distribution.
    (\textbf{d}) Example event conditional samples compared to actual events.  When conditioning on the event defined by the statistic value $E=[C(x)>0]$, the model is able to conditionally sample from the distribution.  Note that while the majority of conditional samples satisfy the event, there are a few stragglers that do not satisfy the event shown by the population the left of $C(x)=0$ in (b) and the purple trajectory in (d).
    }
    \label{fig:conditional_samples_lorenz}
\end{figure*}

\subsection{Convergence of the Gaussian Approximation}\label{subsec:gaussian_approx}

\begin{figure}[t]
    \centering
    \vspace{-2mm}
    \includegraphics[width=\linewidth]{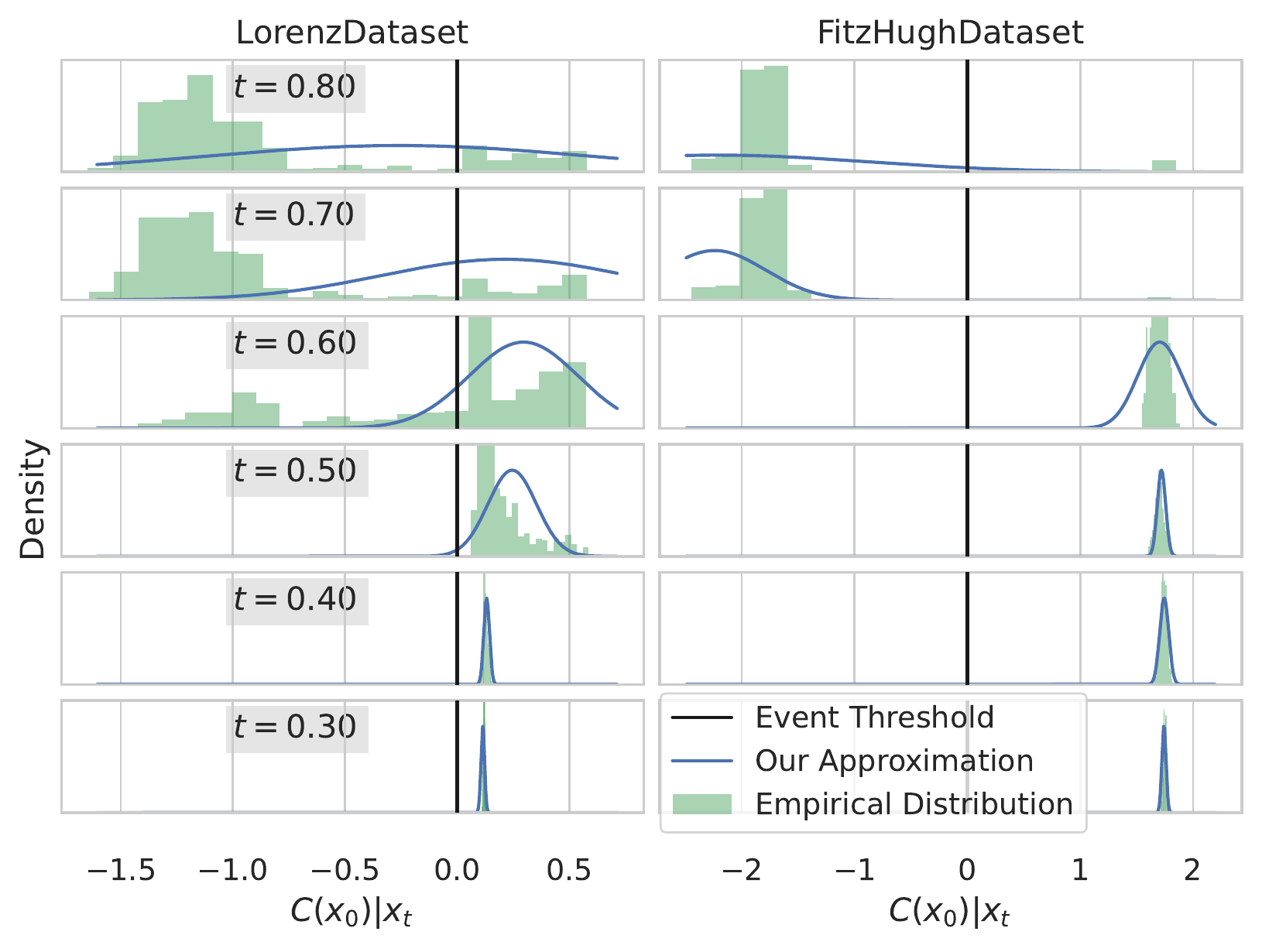}
    \vspace{-1.5mm}
    \caption{The empirical distribution of the event statistic $C(x_0)$ computed over distinct sample paths given the value $x_t$ at a particular noise level for the events chosen in the two datasets. The linearized moment matching approximation loosely guides the generation at high noise levels, and becomes increasingly accurate as the noise level decreases. 
    }
    \label{fig:gaussian_approx}
\end{figure}
Here, we study the accuracy of our approximation $p(C(x_0)|x_t) \approx \mathcal{N}(C(\hat{x}_0),\nabla C^\top \hat{\Sigma}\nabla C)$,
which combines Gaussian moment matching on $x_0|x_t$ with a linearization of the constraints around the mode of the Gaussian. We show that our approximation becomes increasingly exact as the time variable and hence the noise goes to $0$, and the asymptotics of the error with $\sigma_t$ are quantified in \autoref{app:cumulants}.
\begin{theorem}
Suppose $C$ is an analytic function and $\{p(x_t)\}_{t \in [0, 1]}$ is collection of smooth density functions with a smooth dependence on $t$ (associated to $x_t$ in \eqref{eq:sde}), then the random variable $(\nabla C^\top \hat{\Sigma} \nabla C)^{-1/2}\big(C(x_0)-C(\hat{x}_0)\big)$ (conditioned on $x_t$) converges in distribution to a multivariate normal as $t \rightarrow 0$,
\begin{equation}
     (\nabla C^\top \hat{\Sigma} \nabla C)^{-1/2}(C(x_0)-C(\hat{x}_0)) \rightarrow_d \mathcal{N}(0,I).
\end{equation}
\end{theorem}
\textbf{Proof:} See \autoref{app:cumulants}. 

To better understand the transition from high- to low-noise regimes in terms of the convergence of the approximation, we empirically evaluate the distribution $p(C(x_0)|x_t)$ at different points and at various times between $t=1$ and $t=0$ for a trajectory conditioned on the event $C(x_0)>0$. We compare this empirical distribution from sampling with our linearized moment matching approximation in \autoref{fig:gaussian_approx}. For high noise levels ($t>0.5$), the approximation only loosely guides $x_t$ towards satisfying the constraint, however for low noise levels ($t < 0.5$) the approximation becomes increasingly accurate despite the highly nonlinear function $C$. When the noise level is high, fine details of the event gradients $\nabla_{x_t} \log p(E|x_t)$ are less important since they are drowned out by the noise; however, when these fine details matter later on in the generation the constraint approximation becomes increasingly accurate.

\subsection{Computing Marginal Likelihoods of Events}
Our method enables computation of the marginal probability $p(E)$ of the event according to the model, taking into account the different ways the event can happen. In \autoref{app:likelihood_estimation} we investigate estimating these likelihoods when applied to predict whether or not a given initial condition will produce an extreme event $C(x)>0$ in the time window for the neuron firing in the Fitzhugh-Nagumo system. We compare computing the event likelihood directly sampling from the model, using the conditional likelihoods described in \autoref{sec:event_likelihoods}, as well as with importance sampling. For the methods directly using the conditional likelihoods, the difference in likelihoods $\log p(x_0)-\log p(x_0|E)$ is small enough that the two quantities need to be computed extremely precisely, and small errors can produce large errors in the estimated marginal probability $p(E)$. On the other hand, the model has captured the true event likelihood well as evidenced by direct sampling.
\section{Conclusion}
In this work, we successfully build diffusion models for the probabilistic modeling of trajectories of chaotic dynamical systems which are able to capture the dynamics of these systems, and more importantly, provide calibrated uncertainty estimates.
We have developed a probabilistic approximation with theoretical guarantees that enables conditioning the model on nonlinear equality and inequality constraints without retraining the model. With the approach, we are able to sample directly from the tails of the distribution. We discuss limitations of our method in \autoref{app:limitations}.

While in this work we considered ODEs, the applicability of diffusion models extends far beyond that, and we envision a future where a handful of extremely large diffusion models are employed for spatiotemporal weather prediction. We hope that inference time conditioning capabilities will enable querying the model in different ways, such as predicting and anticipating extreme temperatures or adverse events. As the capabilities of these models grow, and the distributions they model become more multifaceted and diverse, exhaustive direct sampling becomes infeasible and retraining to specialize for a given type of conditioning is prohibitive both in terms of compute and data. Additionally, we hope that marginal likelihood computations can be refined in future work, which can be extremely valuable for estimation of extreme events.

\subsection*{Acknowledgements} 
We would like to thank Zhong Yi Wan for insightful discussions during the internship.

Marc Finzi and Andrew Gordon Wilson were partially supported by NSF CAREER IIS-2145492, NSF CDS\&E-MSS 2134216, NSF HDR-2118310, 
NSF I-DISRE 193471, NIH R01DA048764-01A1, NSF IIS-1910266, NSF 1922658 NRT-HDR,
Meta Core Data Science, Google AI Research, BigHat Biosciences, Capital One, and an
Amazon Research Award.

\bibliography{references.bib}
\bibliographystyle{icml2023}

\newpage

\appendix
\section{Tweedie's Covariance}\label{app:tweedies}
Consider the noise relation 
$x = z + \sigma \epsilon$ where $\epsilon \sim \mathcal{N}(0,I)$. 
We write the Gaussian $x|z$ in exponential family form:
\begin{equation}
    p(x|z) = \exp{\big[z^\top T(x)-A(z)\big]}h(x)
\end{equation}
 where $h(x) = e^{-x^\top x/2\sigma^2}/{(2\pi \sigma^2)}^{d/2}$, $A(z) = z^\top z/2\sigma^2$, and $T(x)=x/\sigma^2$ is the sufficient statistic.

Using Bayes rule, $p(z|x) = p(x|z)p(z)/p(x)$, we can rewrite $p(z|x)$ and also express it in exponential family form:
\begin{align}
    p(z|x) & = \exp{\big(z^\top T(x)-A(z)\big)}h(x)p(z)/p(x)\\
    p(z|x) & =\exp{\big(x^\top z/\sigma^2-\log\tfrac{p(x)}{h(x)}\big)}[p(z)e^{-A(z)}] \nonumber \\ 
   & = \exp{\big(x^\top \bar{T}(z)-\bar{A}(x)\big)}\bar{h}(z),
\end{align}
where $\bar{h}(z) = p(z)e^{-A(z)}$, and $\bar{A}(x) = \log\tfrac{p(x)}{h(x)}$, and $\bar{T}(z) = z/\sigma^2$. Despite the fact that the distribution $p(z)$ is not known, $p(z|x)$ is guaranteed to be in exponential family.

A convenient fact is that for exponential families, $\bar{A}(x)$ is the cumulant generating function for $\bar{T}(z)$, and derivatives produce the cumulants:
\begin{align}
    \mathbb{E}[\bar{T}(z)|x] &= \nabla \bar{A}(x) = \nabla \log p(x) - \nabla \log h(x) \\
    \mathrm{Cov}[\bar{T}(z)|x] &=\nabla^2 \bar{A}(x)= \nabla^2 \log p(x) - \nabla^2 \log h(x)
\end{align}
and so forth for higher order cumulants (here $\nabla^2$ denotes the hessian). 

Plugging in $h(x) = e^{-x^\top x/2\sigma^2}/{(2\pi \sigma^2)}^{d/2}$ and $\bar{T}(z) = z/\sigma^2$, and moving the $\sigma^2$ to the other side, we get
\begin{align}
    \mathbb{E}[z|x] &=\sigma^2(\nabla \log p(x) + x/\sigma^2)  \\
    \mathrm{Cov}[z|x] &= \sigma^4 (\nabla^2 \log p(x) + I/\sigma^2)
\end{align}

Finally, with the noise relation in the diffusion models $x_t = s_tx_0+\sigma_t\epsilon$ we can substitute in $z\rightarrow s_tx_0$ and $x\rightarrow x_t$ to get
\begin{align}
    \mathbb{E}[x_0|x_t] &=(x_t+ \sigma^2_t\nabla \log p(x_t))/s_t \\
    \mathrm{Cov}[x_0|x_t] &= \tfrac{\sigma_t^2}{s_t^2}(I+\sigma^2_t \nabla^2 \log p(x_t)),
\end{align}
which proves the relation that we use in the main text.

\section{Convergence of the Moment Matching and Linearization Approximations}\label{app:cumulants}
In this section we show that the linearized moment-matching approximation becomes exact in the limit as the noise scale approaches $0$. First consider the moment matching approximation by itself.
\begin{theorem}
Suppose that  $\{p(x_t)\}_{t \in [0, 1]}$ is a family of smooth probability density functions that depends smoothly on $t$ and  where $x_t$ is given by \eqref{eq:sde}. Then the variable $\hat{\Sigma}^{-1/2}(x_0-\hat{x}_0)$ (conditioned on $x_t$) converges in distribution to a standard multivariate Gaussian, in the limit as $t \rightarrow 0$ (or equivalently as $\sigma_t \rightarrow 0$):
\begin{equation}
    \hat{\Sigma}^{-1/2}(x_0-\hat{x}_0) \rightarrow_d \mathcal{N}(0,I),
\end{equation}
where $\hat{x}_0 = \mathbb{E}[x_0|x_t]$ and $\hat{\Sigma}=\mathrm{Cov}[x_0|x_t]$.
\end{theorem}
\begin{proof}
To determine the convergence of $u:=\hat{\Sigma}^{-1/2}(x_0-\hat{x}_0)$, it is sufficient to show that the expectation converges to $0$, the covariance converges to $I$, and the higher order moments converge to $0$ \citep{janson1988normal}.

To start, we can derive the higher order \emph{cumulants} of the conditional distribution $x_0|x_t$ by taking additional derivatives of the cumulant generating function $\bar{A}$.

Applying the same substitution as in \autoref{app:tweedies}, we obtain the $n$-th order cumulant for $n\ge 3$ given by
\begin{equation}
    k_n(x_0|x_t) := \frac{\sigma_t^{2n}}{s_t^n} \nabla^{\otimes n} \log p(x_t),
\end{equation}
where $\nabla^{\otimes n} = \underbrace{\nabla \otimes \dots \otimes \nabla}_{n}$ and $\otimes$ is the tensor product. 
Now, consider the cumulants of the random variable $u:=\hat{\Sigma}^{-1/2}(x_0-\hat{x}_0)$, which has $0$ mean and a constant scale.
\begin{align}
    k_1(u) &= \mathbb{E}[u] = 0\\
    k_2(u) &=\mathrm{Cov}[u] = \hat{\Sigma}^{-1/2}\hat{\Sigma}\hat{\Sigma}^{-1/2} = I\\
    k_n(u) &= \sigma_t^n (\tfrac{\sigma_t}{s_t}\hat{\Sigma}^{-1/2}\nabla)^{\otimes n} \log p(x_t) \quad \mathrm{for} \ {n\ge3}
\end{align}
Investigating the limiting behavior of $k_n(u)$, we can make use of two important facts: firstly
\begin{equation*}
    \lim_{t \rightarrow 0}\tfrac{\sigma_t}{s_t}\hat{\Sigma}^{-1/2} = \big(\lim_{\sigma_t \rightarrow 0} I+\sigma_t^2\nabla^2\log p(x_t) \big)^{-1/2} = I
\end{equation*}
and secondly $\lim_{t\rightarrow 0} \nabla^{\otimes n} \log p(x_t) = \nabla^{\otimes n} \log p(x_0)$ since $\log p(x_0)$ is smooth.

Therefore we see that $\lim_{t\rightarrow 0} k_n(u) = \nabla^{\otimes n} \log p(x_0) (\lim_{t\rightarrow 0} \sigma_t^n) = 0$, the higher order cumulants converge to $0$ at a rate of $\sigma_t^n$, whereas the mean and variance are fixed.

Therefore according to \citep{janson1988normal}, $u$ converges in distribution to a multivariate normal $u \rightarrow_d \mathcal{N}(0,I)$. 
\end{proof}

Next we consider the full approximation, including the linearization of the constraint as described in \autoref{subsec:nonlinear}.
The constraint is linearized from the Taylor expansion
\begin{align*}
    C(x_0) &= C(\hat{x}_0) \\ &+ \nabla C^\top (x_0-\hat{x}_0) + (x_0-\hat{x}_0)^\top \nabla^2 C(x_0-\hat{x}_0)/2,
\end{align*}
(with additional higher order terms omitted for brevity).
For simplicity we have used notation for scalar function $C$, but the result below holds for vector $C$ analogously. Now we state the convergence result for our approximation
\begin{equation}
    p(C(x_0)|x_t) \approx \mathcal{N}\big(C(\hat{x}_0),\nabla C^\top \hat{\Sigma} \nabla C\big).
\end{equation}

\begin{theorem}
Suppose that $C$ is analytic and $p(x_t)$ is a smooth function then as $\sigma_t \rightarrow 0$, the random variable $(\nabla C^\top \hat{\Sigma} \nabla C)^{-1/2}\big(C(x_0)-C(\hat{x}_0)\big)$ converges in distribution to a multivariate normal:
\begin{equation}
     (\nabla C^\top \hat{\Sigma} \nabla C)^{-1/2}(C(x_0)-C(\hat{x}_0)) \rightarrow_d \mathcal{N}(0,I).
\end{equation}
\end{theorem}

\begin{proof}
To start, we note that $\lim_{\sigma_t \rightarrow 0}\tfrac{s_t^2}{\sigma_t^2}\nabla C^\top \hat{\Sigma}\nabla C  = \nabla C^\top \nabla C$, so if we can prove that for the random variable
    $v = \tfrac{s_t}{\sigma_t} \big(C(x_0)-C(\hat{x}_0)\big)$ converges to $\mathcal{N}(0, \nabla C^\top \nabla C)$ then we have proven the claim.
    
For convenience, define $A = (s_t/\sigma_t)\hat{\Sigma}^{1/2}$, keeping in mind that $\lim_{\sigma_t \rightarrow 0} A = I$
    
Recalling the random variable $u=\hat{\Sigma}^{-1/2}(x_0-\hat{x}_0)$, we can rewrite $v$ using the Taylor series as
\begin{equation}
    v= \nabla C^\top Au + \tfrac{1}{2}\big(\tfrac{\sigma_t}{s_t}\big) u^\top A^\top \nabla^2 CAu + O(\big(\tfrac{\sigma_t}{s_t}\big)^2).
\end{equation}

Notably, $A$, $u$, $u^\top \nabla^2 Cu$, and higher order terms converge to a fixed scale as $\sigma \rightarrow 0$ (since $C$ is assumed to be twice continuously differentiable 
and $u$ converges to a normal).

Writing out the cumulants of this random variable we see a similar pattern as before:
\begin{align*}
    k_1(v) &= \mathbb{E}[v] = O(\big(\tfrac{\sigma_t}{s_t}\big))\\
    k_2(v) &=\mathrm{Cov}[v] = \nabla C^\top A\mathrm{Cov}[u] A^\top\nabla C+ O(\big(\tfrac{\sigma_t}{s_t}\big))\\
    k_n(v) &= \sigma_t^n (\nabla C^\top A\nabla)^{\otimes n} \log p(x_t) + O(\big(\tfrac{\sigma_t}{s_t}\big))\quad \mathrm{for} \ {n\ge3}.
\end{align*}

Note that $\lim_{\sigma_t \rightarrow 0} (\nabla C^\top A\nabla)^{\otimes n} \log p(x_t) = (\nabla C^\top \nabla)^{\otimes n} \log p(x_0)$ since $p$ is smooth. In the limit as $t \rightarrow 0$, the cumulants become

\begin{align*}
    \lim_{t \rightarrow 0} k_1(v) &= 0\\
    \lim_{t \rightarrow 0} k_2(v) &= \nabla C^\top \nabla C\\
    \lim_{t \rightarrow 0} k_n(v) &= 0 \quad \mathrm{for} \ {n\ge3}.
\end{align*}

Therefore, according to \citep{janson1988normal}, $v$ converges in distribution to $\mathcal{N}(0,\nabla C^\top \nabla C)$.
\end{proof}

In contrast with the mere moment matching Gaussian approximation the convergence rate is considerably slower however, with higher order cumulants only decaying as $O(\big(\tfrac{\sigma_t}{s_t}\big))$ and depending on the smoothness of $C$. Nevertheless, as the noise scale gets smaller, the combined linearization and moment matching approximation (informally stated)
\begin{equation}
    p(C(x_0)|x_t) \underset{t\rightarrow 0}{\rightarrow} \mathcal{N}\big(C(\hat{x}_0),\nabla C^\top \hat{\Sigma} \nabla C\big)
\end{equation}
becomes exact.

\section{Marginal Likelihood Estimation}\label{app:likelihood_estimation}

In estimating the marginal event likelihood $p(E)$, there are multiple ways of extracting this quantity from the model. The simplest, and least scalable to extremely low likelihood events is extensively sample from the model and compute the fraction which satisfy the event $p(E) = \mathbb{E}_{x\sim p(x)}[\mathds{1}{[x\in E]}]$. Second is to use the method we introduce in \autoref{eq:marginal}, and compute the average increase in the log likelihood 
\begin{equation}\label{eq:marginal_estimator}
    p(E) = \exp{(\mathbb{E}_{x \sim p(x|E)}[\log p(x)-\log p(x|E)])}.
\end{equation}

 \begin{figure}[h!]
    \centering
    \vspace{-4mm}
    \includegraphics[width=0.3\textwidth]{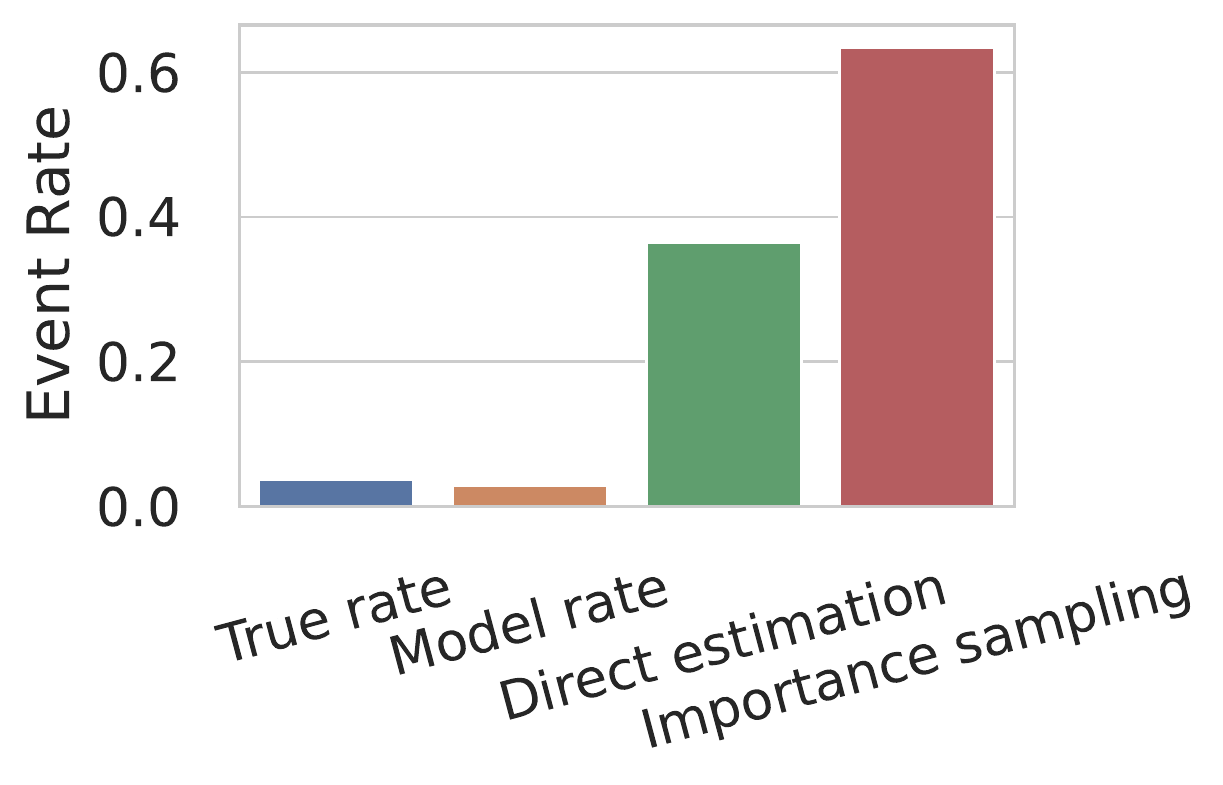}
    \vspace{-1.5mm}
    \caption{Calibration of different estimators of the marginal event likelihood for Fitzhugh-Nagumo extreme event. True rate is computed over sampling from the original dataset and model rate is computed over samples from the diffusion model. Direct estimation uses \autoref{eq:marginal_estimator} and importance sampling uses the importance sampled version of the estimator, both are computed using $3$ samples from $x\sim p(x|E)$. Likelihood based estimators are not well calibrated to the true rate of ~1/30, but the model rate is.}
    \label{fig:event_probabilities}
\end{figure}

 The potential downside of this approach is that its validity depends on the quality of the approximation used to compute $p(x|E)$. Alternatively, we can instead drop this requirement and view $p(x|E)=q(x)$ merely as a strong proposal distribution for importance sampling. Sampling from $q$, $p(E) = \int \mathds{1}{[x\in E]}p(x)dx = \mathbb{E}_{x\sim p(x|E)}[p(x)/p(x|E)]$ giving a very closely related estimator to \autoref{eq:marginal_estimator}, but that has no requirements on $p(x|E)$ other than it covers the event space. In \autoref{fig:event_probabilities}, we evaluate the predictions of these three estimators on the FitzHugh-Nagumo neuron spiking event, and compare to the ground truth event rate.
 
 While the sampling based method $p(E) = \mathbb{E}_{x\sim p(x)}[\mathds{1}{[x\in E]}]$ (Model rate in the figure) well approximates the true model event rate, unfortunately neither of the likelihood based methods (Direct estimation and Importance sampling) for estimate $p(E)$ appear to be calibrated when estimated over a small number of samples $p(x_0|E)$, even though theoretically they should produce the value consistent with the model distribution. We suspect this has to do with the accumulation of numerical errors in the estimation of the Jacobian log determinants for $p(x|E)$ which needs to be estimated very precisely because $\log p(x|E)$ is typically on the order of $1000$ and should differ from $\log p(x)$ only by a few tenths of a percent.

\section{Probabilistic Origin of Constraint Projection}\label{app:linear_constraints_projection}
In this section we investigate how the additional constraint projection steps used in \citet{chung2022come,chung2022improving} can arise in our probabilistic framework when applied to linear constraints, but performing one additional approximation. Consider the goal of imposing the set of linear constraints $Cx = y$ onto samples from the diffusion model for a given constraint matrix $C \in \mathbb{R}^{r\times d}$ and $y\in \mathbb{R}^r$. We seek to sample $x \sim p(x|Cx=y)$ using a diffusion models.

Without loss of generality, we can orthogonalize the linear constraints. Decomposing $C$ with the SVD into the right nullspace of $C$ and its complement:
\begin{equation}
    C = \begin{bmatrix}U & V \end{bmatrix}\begin{bmatrix} \Sigma & 0 \\ 0 & 0\\\end{bmatrix}\begin{bmatrix}Q^\top \\P^\top \end{bmatrix}.
\end{equation}
Here the matrix $Q \in \mathbb{R}^{d\times r}$ and $P \in \mathbb{R}^{d \times d-r}$ correspond to the subspace of $x$ that is determined by the constraint, and the orthogonal complement which is linearly independent of the constraint.
We can now rewrite the constraint $Cx=y$ as $U\Sigma Q^\top x=y$ or equivalently: $Q^\top x=u$ where $u = \Sigma^{-1}U^\top y$.
 
We can decompose $x$ in these two components: its projection onto the row space of $C$ and onto its orthogonal complement, $x = Qu + Pv$, where $v := P^\top x$. In contrast with the derivation in \autoref{sMethod}, we will split up the diffusion process into these two subspaces. 
In order to sample conditionally, we need the conditional scores
\begin{equation}
    \nabla_{x_t} \log p(x_t|C x_0=y) = \nabla_{x_t} \log p(x_t|Q^\top x_0=u_0).
\end{equation}
We can express the gradients with respect to $x_t$ as the sum of the projected gradients with respect to $u_t$ and with respect to $v_t$ using the chain rule: $\nabla_{x_t} = Q\nabla_{u_t} +P\nabla_{v_t}$.
To see this, let $O=[Q, P] \in \mathbb{R}^{d\times d}$ be the concatenation of $Q$ and $P$ that forms a full rank orthogonal matrix ($O^\top O=I$). Let 
\begin{equation}
z_t: = \left [ \begin{array}{c}
     u_t  \\
     v_t
\end{array} \right] =  O^\top x_t,
\end{equation}
which can be inverted to get $Oz_t = x_t$. Applying the chain rule, it follows that
\begin{equation}
    \nabla_{x_t} = O\nabla_{z_t} = Q\nabla_{u_t} +P\nabla_{v_t}.
\end{equation}

Therefore, we can split up the conditional scores into gradients with respect to the two variables $u_t$ and $v_t$:
\begin{align}
    \nabla_{x_t} \log p(x_t|u_0) &= Q\nabla_{u_t} \log p(x_t|u_0) \\
                                 &+ P\nabla_{v_t} \log p(x_t|u_0).
\end{align}

So far, this equation merely expresses \autoref{sMethod} in a different form. 
The additional approximation needed to produce the explicit constraint projection is to replace or approximate 
$\nabla_{u_t} \log p(x_t|u_0)$ with $\nabla_{u_t} \log p(u_t|u_0)$. Despite being  closely related they are different:
$\nabla_{u_t} \log p(x_t|u_0) = \nabla_{u_t} \log p(u_t|u_0) + \nabla_{u_t} \log p(v_t|u_t,u_0)$. If we use this approximation, then this first term can be easily computed from the forward noising process
$u_t \sim \mathcal{N}(s_tu_0,\sigma_t^2 I)$ and therefore 
\begin{equation}
    \nabla_{u_t} \log p(u_t|u_0) = \frac{s_tu_0-u_t}{\sigma_t^2}.
\end{equation}

For the second term, one may verify from Bayes rule that
\begin{align}
    p(x_t|u_0) &= p(v_t|u_t)p(u_0|x_t)p(u_t|u_0)/p(u_0|u_t) \\
    \nabla_{v_t} \log p(x_t|u_0) &=  \nabla_{v_t} \log p(v_t|u_t) + \nabla_{v_t} \log p(u_0|x_t),
\end{align}
as all but these two terms do not depend on $v_t$.
The first term can be identified as simply the scores of the unconditional network projected onto the unknown subspace:
\begin{equation}
    \nabla_{v_t} \log p(v_t|u_t) = \nabla_{v_t} \log p(v_t,u_t) = P^\top s_\theta(x_t,t).
\end{equation}

Performing moment matching on $p(u_0|x_t)$ like before, we can approximate the distribution with a multivariate Gaussian which shares its true mean and covariance:
\begin{align}
    p(u_0|x_t) &\approx \mathcal{N}(Q^\top\hat{x}_0,Q^\top \hat{\Sigma}Q). 
\end{align}
With the combination of the two approximations, the conditional scores become:
\begin{align}
    \nabla_{x_t} \log p(x_t|u_0) &= Q\big[s_tu_0-u_t\big]/\sigma_t^2 + P\nabla_{v_t}\log p(x_t|u_0)
\end{align}
where
\begin{align}
    \nabla_{v_t}\log p(x_t|u_0) &\approx P^\top s_\theta(x_t,t) \nonumber \\ &+\nabla_{v_t}\log \mathcal{N}(u_0; Q^\top\hat{x}_0,Q^\top \hat{\Sigma}Q)
\end{align}

Substituting $\hat{\Sigma} = 2\alpha I$ to match \citet{chung2022improving} (which neglects the scaling with $\sigma_t/s_t$), and applying the chain rule to convert derivatives with respect to $v_t$ to derivatives with respect to $x_t$,
\begin{equation}
    \nabla_{v_t}\log \mathcal{N}(u_0; Q^\top\hat{x}_0,2\alpha I) = -\alpha P^\top\nabla_{x_t}\|Q^\top \hat{x}_0 -u_0\|^2.
\end{equation}

Finally, assembling the terms, we have
\begin{align}
    \bar{s}(x_t,t) &= Q\big[s_tu_0-u_t\big]/\sigma_t^2  \nonumber \\
    & + PP^\top[s(x_t,t)-\alpha \nabla_{x_t}\|Q^\top \hat{x}_0 -u_0\|^2],
\end{align}
where $\bar{s}(x_t,t)$ denotes the approximation for $\nabla_{x_t} \log p(x_t|u_0)$ when using these additional assumptions. The three terms can be understood as: (1) the analytically known diffusion in the known subspace, (2) the projection of the unconditional score function to the unknown subspace, and (3) the Gaussian correction term projected into the unknown subspace.

We can identify each of the terms in this equation directly with equations 7 and 8 in \citet{chung2022improving} where the terms are renamed as follows. Their variable names correspond to the following quantities:
$y \leftarrow u_0 = R^{-T}y$, $P\leftarrow Q^\top$, and $A=I-P^\top P \leftarrow I-QQ^\top = PP^\top$. While here we have the results expressed in terms of the score function rather than a denoising step, the two are consistent. If framed as a denoising step, $Q\big[s_tu_0-u_t\big]/\sigma_t^2$ would become the $b$ term: $b \leftarrow Q\epsilon$ where $\epsilon \sim \mathcal{N}(s_tu_0,\sigma_t^2I)$. The only difference is the $W$ matrix which is not fully specified in their method (they only provide a couple examples where it is suggested what it should be).

Therefore we see that while the constraint projections of \citet{chung2022improving} is not \emph{equivalent} to our derivation in \autoref{sMethod} for the linear case with identity covariance, it naturally arises when making the additional replacement
$\nabla_{u_t} \log p(x_t|u_0) \mapsto \nabla_{u_t} \log p(u_t|u_0)$. Intuitively speaking, this additional assumption assumes that we can evolve the noised version of the known values $u_t$ without considering the unknown values.

\section{Conditional Score Convergence in Continuous Time} \label{app:convergence}
 \begin{figure}[h!]
    \centering
    \includegraphics[width=0.4\textwidth]{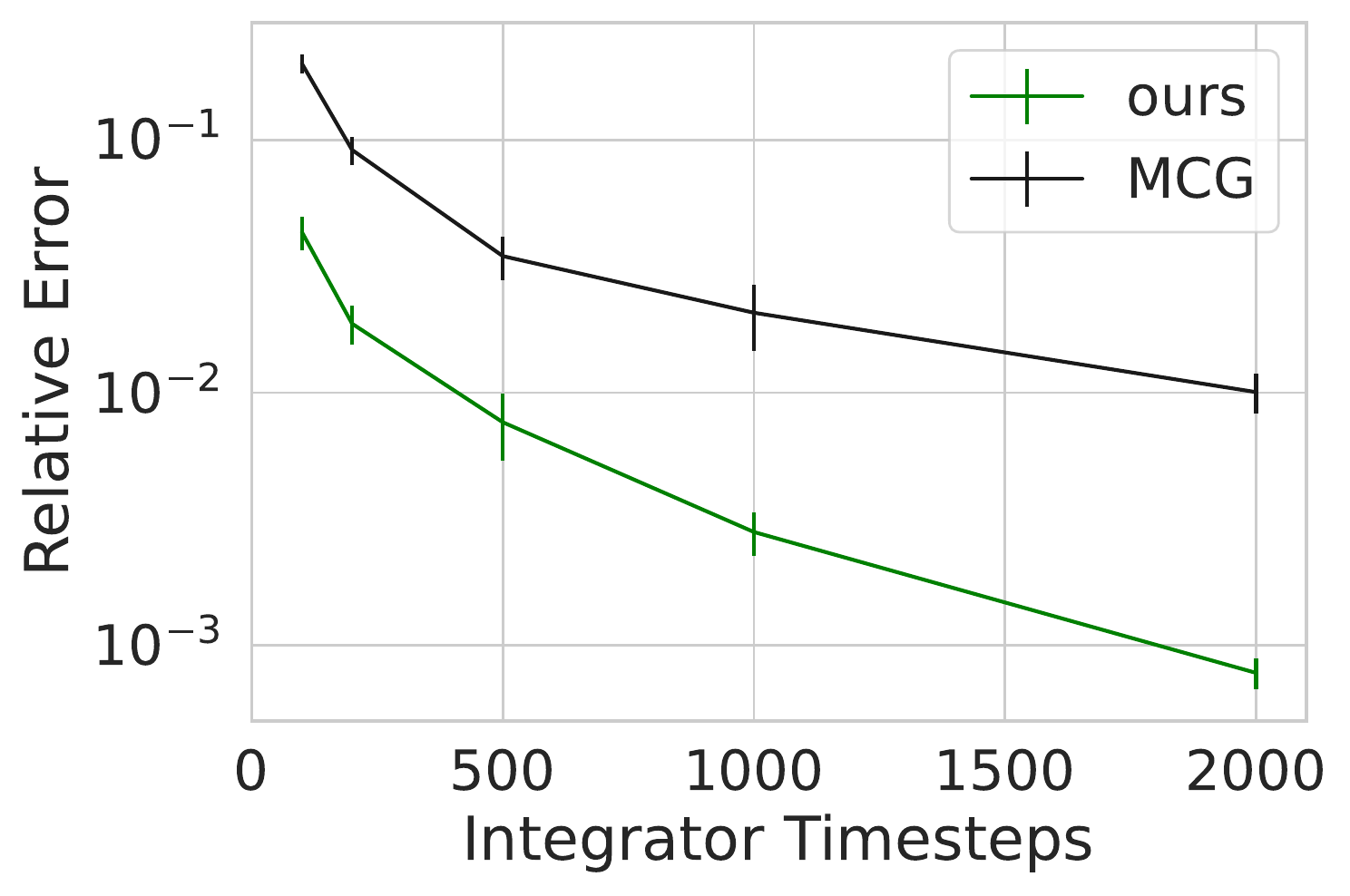}
    \vspace{-1.5mm}
    \caption{Convergence of conditional samples with varying numbers of ODE integrator steps. Due to the missing scaling factors of MCG \citep{chung2022improving}, the method requires many more integration steps to converge.}
    \label{fig:integrator_convergence}
\end{figure}
When considering equality constraints, both Manifold Constrained Gradients (MCG) \citep{chung2022improving} (linear constraints only) and Diffusion Posterior Sampling  (DPS) \citep{chung2022diffusion} (nonlinear equality constraints with or without noise)  can be used for inference time conditional sampling. While these two approaches are effective in this setting, they are not well suited to the continuous time ODE formulation, a requirement for computing likelihoods through the change of variables formula.

The difference can be summarized on a one dimensional constraint $c^\top x=y$, for which the approach of \citet{chung2022diffusion} would give the conditional scores as:

\begin{equation}\label{eq:MCG}
   \nabla_{x_t}\log p(x_t|y)  = s_\theta(x_t,t) - \alpha \nabla_{x_t}\|c^\top \hat{x}_0-y\|^2,
\end{equation}

whereas our method yields the scores 
\begin{equation}
   \nabla_{x_t}\log p(x_t|y)  = s_\theta(x_t,t) - \nabla_{x_t}\frac{\|c^\top \hat{x}_0-y\|^2}{\sqrt{c^\top \hat{\Sigma} c}}.
\end{equation}

The $\hat{\Sigma}$ matrix scales as $\sigma_t^2/s_t^2$ which varies by many orders of magnitude. While MCG and DPS include the tuneable parameter $\alpha$, it has one fixed value, and cannot match the $\sigma_t^2/s_t^2$ scale inside $\hat{\Sigma}$. As a result, when incorporating \autoref{eq:MCG} into adaptive step size integrators, the scales at different times of the integration will either be too small or too large in comparison to $s_\theta(x_t,t)$. We see this in practice that with MCG or DPS and an adaptive step size integrator, the integration fails to converge.

In \autoref{fig:integrator_convergence} we demonstrate that this is a problem even for fixed step size integrators too, by evaluating the convergence of the an ODE integrated trajectory sampling from a linear equality constraint with $c=\mathbf{1}/d$ and $y=0.08$ on the FitzHugh model. We measure the relative error between the conditional sample produced with $4000$ uniformly spaced ODE integrator steps and the conditional sample produced with smaller numbers of integrator steps. As shown in the figure, our conditional score functions lead to a much more rapidly converging solution.

\section{Dataset Construction}\label{app:datasets}
For all datasets, we generate $4000$ trajectories discretized to $60$ timesteps to be used for training, and hold out an additional $500$ trajectories for testing.

\textbf{Lorenz Attractor}
We generate trajectories from the classical dynamics
\begin{align}
    \dot{x} &= 10 (y-x) \\
    \dot{y} &= x(28-z) - y \\
    \dot{z} &= xy - (8/3) z 
\end{align}
Writing these components as the vector $\mathbf{x} = [x,y,z]$, these dynamics can be written as $\dot{\mathbf{x}} = F(\mathbf{x})$ for the above function $F$. Since these dynamics range over the scale $(-60,60)$ we rescale the dynamics by defining a rescaled version of the Lorenz dynamics: $\tilde{F}(\mathbf{x}) = F(20\mathbf{x})/20$, which preserves the dynamics but scales down the values into the range $(-3,3)$ which makes the learning of the diffusion model and Neural ODE more stable. We sample the initial conditions from the standard normal $\mathbf{x}(0) \sim \mathcal{N}(0,I)$, we integrate for a total of $10$ seconds and then discard the first $3$ seconds as burn-in to approximately reach the strationary distribution. The trajectories are then discretized to $60$ evenly spaced timesteps when training the diffusion model.

\textbf{FitzHugh Nagumo}
We follow \citet{farazmand2019extreme} in their choice of parameters to produce the rare events. We sample initial conditions from $\textbf{x}(0) \sim .2\mathcal{N}(0,I)$ on the $4$-dimensional system. We integrate for $4000$ units of time and discard the first $1500$ units of time for burn-in, which we find to be sufficient for the distribution to approximately reach the stationary distribution.

\textbf{Double Pendulum}
We use Hamiltonian from \citet{finzi2020simplifying} for the double pendulum in angular coordinates, and then integrate the Hamiltonian dynamics, with a final postprocessing step of converting the canonical momenta back into angular velocities. We set the mass and length parameters to $1$ for simplicity and integrate for $\tau=30$ seconds. We sample the initial from Gaussians. For the angles from $\theta(0)$ we use standardard deviation $\sigma=1$ and the initial momenta we use $\sigma=.2$ for $p_{\theta_1}$ and $\sigma=.3$ for $p_{\theta_2}$ .

When fitting the system with the diffusion model, we only proved the angle values, and we first embed the two angles into the plane using $\cos\theta_i$ and $\sin \theta_i$ for the two angles $i=1,2$. We performed this additional featurization since some trajectories wrap the angles around many times, and the large angle values can cause problems.
For the Neural ODE, we trained it in the usual way feeding in both $\theta_i$ and $\dot{\theta}_i$ as done in \citet{finzi2020simplifying}.

When training the NeuralODE \citep{chen2018neural}, for each of the systems, we chunk the 4000 length 60 trajectories into a total of 24000 chunks of length 10, which is a standard practice to improve the stability and efficiency of training (see e.g. \citep{finzi2020simplifying}. We train using the $L_1$ loss on the prediction error on the 10 evaluation points for each trajectory. The NeuralODE uses a simple 3-layer MLP with swish nonlinearities and $128$ hidden units.
\section{Training and Hyperparameters} \label{app:hyperparams}
When training, we sample diffusion times $t$ on a shifted grid following \citet{kingma2021variational} for reduced training times. We optimize the score matching loss with ADAM \citep{kingma2014adam} for $10000$ epochs with lr $10^{-4}$ and bs $500$. We use the variance exploding schedule for $\sigma_t$ and $s_t$, as described in \citet{song2020score}. We evaluate all models using the exponential moving average of the parameters at the final epoch, where the EMA period is $2000$ epochs.

\section{Model} \label{app:architecture}
We use a convolutional UNet architecture similar to \citet{saharia2022photorealistic} but scaled down, with the self attention layers removed (we did not find them helpful at this scale), with a modified embedding to handle continuous input times (the method used in \citet{song2020score}, and with $2$D convolutions replaced with $1$D convolutions. At a high level, the architecture can be summarized with the following table with value $c=32$, and $d$ represents the input and output dimension.

\begin{table}[!ht]
\centering
\label{table:architecture}
\begin{tabular}{cc}
\multicolumn{2}{c}
{\bf{Convolutional UNet Architecture:} }\\
\hline
ResBlock($c$) & x4\\
Downsample(2) \\
ResBlock($2c$) & x8\\
Downsample(2) \\
ResBlock($4c$) & x8\\
\hline
SkipResBlock($4c$) & x8 \\
Upsample(2) \\
SkipResBlock($2c$) & x8 \\
Upsample(2) \\
SkipResBlock($c$) & x4\\
\hline
Conv($128$)\\
Conv($d$)
\end{tabular}
\end{table}

The resblock follows the standard structure, but with GroupNorm instead of BatchNorm, using swish nonlinearities, and 1D convolutions.
\begin{table}[!ht]
\centering
\label{table:resblock}
\begin{tabular}{ l }
\noalign{\medskip}
\bf{ResBlock(c):} \\
\hline
GroupNorm(groups=c//4)\\
Swish\\
Conv(channels=3, ksize=3)\\
GroupNorm(groups=c//4)\\
Swish\\
Conv(channels=3, ksize=3)\\
SkipConnection\\
\hline
\end{tabular}
\end{table}
SkipResBlocks utilize skip connections with the corresponding residual block in the downsampling portion of the network, as done in typical diffusion UNets such as in \citet{karras2022elucidating}.

For gradients of the Gaussian CDF function, we instead use the logit approximation $\Phi(z) \approx \sigma(1.6z)$ which is more numerically stable.

\section{Computing Likelihoods}
\label{app:likelihood}
While the two probabilities $p(x_0)$ and $p(x_0|E)$ can in principle be computed by integrating the probability flow ODEs (\autoref{eq:prob_flow}) forwards in time with the continuous change of variables formula used in FFJORD \citep{grathwohl2018ffjord,song2020score}, the variance of the likelihood estimator is large (requiring a large number of probe variables) and the integration times can be very long. The large variance is typically not a problem when averaging to form the average log likelihoods of the entire dataset (which reduces the variance) and also because only a small number of significant digits are required for the metrics. On the other hand, for our purposes where we need to compute the log likelihood on a single data point very precisely, the estimator variance is too large.

Instead of using the continuous change of variables approach, we instead use a fixed timestep 2nd order Heun integrator to control the compute time, and then consider this integrator just as some deterministic and invertible function $x_{t+\Delta t} = H(x_t)$. While one cannot apply the continuous change of variables formula with fixed timesteps, we can instead simply compute the Jacobian of the ODE integrator. With this Jacobian, we can compute the likelihoods exactly without sampling and even when the timesteps are large using the relation:
\begin{equation}
    \log p_{t=0}(x_0|E) = \log p_{t=1}(x_1) + \sum_i \log \mathrm{det} DH(x_{t_i}),
\end{equation}
where $p(x_1) = \mathcal{N}(0,\sigma_1^2)$. We compute the Jacobians exactly which removes the need for sampling, however it's possible to replace this computation with Jacobian vector products even for the discrete trajectory by using the matrix logarithm expansion and Russian roulette estimator used for Residual Flows \citep{chen2019residual}.

\section{Limitations}\label{app:limitations}
We introduced a general and principled method for sampling conditionally on user defined events according to nonlinear equality and inequality constraints. However, our method has several limitations.

\textbf{Scope}: Our method was designed for deterministic events, and while it is easy to extend to noisy measurements so long as they have Gaussian observation noise, for non Gaussian observation noise our approximation will not be valid. Furthermore, more complex set based event constraints (such as the values being contained in a given region) may not be easily expressed as equality or inequality constraints that are supported by our method. 

\textbf{Computational cost}: Our second order approximation requires computing Jacobian vector products with the diffusion score function which can be expensive. If this cost is prohibitive, we recommend using the cruder but still principled approximation
\begin{equation}
    p(C(x_0)|x_t) \approx \mathcal{N}(C(\hat{x}_0),(\sigma_t/s_t)^2\nabla C^\top \nabla C),
\end{equation}
when using our method. 
However, the costs of this approximation are still quadratic in the number of constraints $r$. In situations where the constraints are very high dimensional and even computing $\nabla C^\top \nabla C$ is not possible, we recommend falling back to the diagonal approximation $p(C(x_0)|x_t) \approx \mathcal{N}(C(\hat{x}_0),\lambda (\sigma_t/s_t)^2I)$, a version of \citet{chung2022diffusion} requiring the tunable $\lambda$ parameter.

Likelihood evaluation is even more expensive, requiring computation of the Jacobian log determinant of the ODE integration step. This cost is $O((md)^3)$ and is prohibitive for large signals such as images. In future work this can be addressed such as by using the Russian roulette estimator from \citet{chen2019residual} which will reduce the cost to $O(md)$. Finally, the marginal likelihood computation $p(E)$ via the likelihoods is difficult to estimate accurately due to the differencing of two similar values, we hope this can be addressed in future work.

\vfill

\end{document}